\documentclass{article}
\usepackage{natbib}
\setcitestyle{numbers,square}
\usepackage{amsmath}
\usepackage{amsfonts} 
\usepackage{amsthm}
\usepackage{mathtools}
\usepackage[linesnumbered,ruled]{algorithm2e}
\renewcommand{\KwData}{ \textbf{Input:}} 
\usepackage[preprint]{neurips_2024}
\usepackage{graphicx}
\usepackage{subfigure}
\usepackage[utf8]{inputenc} 
\usepackage[T1]{fontenc} 
\usepackage{hyperref} 
\usepackage{url} 
\usepackage{booktabs} 
\usepackage{nicefrac} 
\usepackage{microtype} 
\usepackage{xcolor} 
\usepackage{array,tabularx}
\usepackage{enumitem}
\setitemize[1]{itemsep=0pt,partopsep=0pt,parsep=2pt,topsep=-3pt}

\title{Minimizing UCB: a Better Local Search Strategy in Local Bayesian Optimization}
\author{$\textbf{Zheyi Fan}^{1,2},\ \textbf{Wenyu Wang}^{3}\ ,\ \textbf{Szu Hui Ng}^{3}\ ,\ \textbf{Qingpei Hu}^{1,2}$\\
$^{1}$Academy of Mathematics and Systems Science, Chinese Academy of Sciences, China\\
$^{2}$School of Mathematical Sciences, University of Chinese Academy of Sciences, China\\
$^{3}$Department of Industrial Systems Engineering \& Management, National University of Singapore, Singapore\\
$^{1,2}$\texttt{\{fanzheyi,qingpeihu\}@amss.ac.cn}, $^{3}$\texttt{\{wenyu\_wang,isensh\}@nus.edu.sg}\\ }
\date{February 2024}

\begin{document}

\maketitle
\begin{abstract}
    Local Bayesian optimization is a promising practical approach to solve the high dimensional black-box function optimization problem. Among them is the approximated gradient class of methods, which implements a strategy similar to gradient descent. These methods have achieved good experimental results and theoretical guarantees. However, given the distributional properties of the Gaussian processes applied on these methods, there may be potential to further exploit the information of the Gaussian processes to facilitate the BO search. In this work, we develop the relationship between the steps of the gradient descent method and one that minimizes the Upper Confidence Bound (UCB), and show that the latter can be a better strategy than direct gradient descent when a Gaussian process is applied as a surrogate. Through this insight, we propose a new local Bayesian optimization algorithm, MinUCB, which replaces the gradient descent step with minimizing UCB in GIBO \cite{muller2021local}. We further show that MinUCB maintains a similar convergence rate with GIBO. We then improve the acquisition function of MinUCB further through a look ahead strategy, and obtain a more efficient algorithm LA-MinUCB. We apply our algorithms on different synthetic and real-world functions, and the results show the effectiveness of our method. Our algorithms also illustrate improvements on local search strategies from an upper bound perspective in Bayesian optimization, and provides a new direction for future algorithm design.
\end{abstract}

\section{Introduction}\label{sec:introduction}
   Bayesian Optimization \cite{frazier2018tutorial} is one of the most well-known black box function optimization methods, where objectives can be extremely expensive to evaluate, noisy, and multimodal. The high efficiency of  Bayesian Optimization in finding global optima leads to the widespread application in various research fields, such as hyperparameter tuning \cite{gardner2014bayesian,gelbart2014bayesian,snoek2012practical}, neural architecture search \cite{shen2023proxybo}, chemical experiment design \cite{guo2023bayesian}, reinforcement learning \cite{turchetta2020robust}, aerospace engineering \cite{lam2018advances}. However, the performance of Bayesian optimization is limited by the input dimension $d$, as the theoretical regret bound grows exponentially with input dimension \cite{shekhar2021significance}. This difficulty hinders the application of Bayesian optimization when the actual dimension of problem exceeds $20$ \cite{frazier2018tutorial}.

    There are various methods that have been proposed to handle this difficult task, including works that rely on some assumptions on the model structure, such as the assumption that the majority of the variables have no effect \cite{eriksson2021high} or the kernel satisfies an additive structure \cite{gardner2017discovering}. Local Bayesian optimization methods, which focus on finding a local optima (instead of the global one), have also been a popular (and less restrictive) compromise to manage the curse of dimensionality. Representative methods of these include those based on local trust region methods \cite{eriksson2019scalable}, local latent space \cite{maus2022local}, and approximated gradient methods \cite{muller2021local,nguyen2022local,wu2023behavior}. Among them the approximate gradient method has demonstrated strong performance in practical applications compared with other methods. The approximate gradient method can be described as a two-stage algorithm, which loops through the following two processes: first sample points to decrease the uncertainty of the local area according to a local exploration acquisition function, and then moving to the next point with a trustworthy high reward through a local exploitation acquisition function. M{\"u}ller et al. \cite{muller2021local} applied the idea of gradient descent and first proposed the GIBO algorithm, which was designed to alternate between sampling points to minimize the posterior variance of the gradient at a given location, and then moving in the direction of the expected gradient. Nguyen et al. \cite{nguyen2022local} proposed MPD, which improved the local exploitation acquisition function through defining the descent direction by maximizing the probability of descent, and designed corresponding local exploration acquisition function to match this strategy.
    Wu et al. \cite{wu2023behavior} further provided the detailed proof on the local convergence of GIBO with a polynomial convergence rate, for both the noiseless and noisy cases.

    Although the approximated gradient method has been shown to be practical in dealing with high-dimensional problems, there may be still some room to potentially improve on the current methods. We motivate the ideas for improvement with the following two questions. 
    
    1) We observe that GIBO only utilizes the posterior distributions of the gradient at a point, which will ignore most of the information provided by Gaussian process surrogate in the region, which may lead to an inefficient descent. MPD attempts to make better use of Gaussian processes by performing multi-step descent, but this strategy can exhibits numerical instability and may lead to suboptimal performance of the algorithm (as seen in Section \ref{experiment}). This motivates us to think: is there a better local exploitation acquisition function that can ensure the algorithm fully utilizes Gaussian process information, and also guarantee the convergence to local optima points? 

    2) Do these acquisition functions necessarily need to depend on accurate gradient estimates at a point, or are there other acquisition functions that can improve the efficiency and still ensure local convergence?


    In this paper, we attempt to answer the above two questions through our two new local Bayesian optimization algorithms. To address the first question, we first develop the relationship between the step of the gradient descent method and minimizing the Upper Confidence Bound (UCB). When the Gaussian process is applied as the surrogate model, minimizing the UCB can usually achieve a point with a higher reward than simply doing gradient descent. Motivated by this idea, we propose our first algorithm, Local Bayesian Optimization through Minimizing UCB (MinUCB), which replaces gradient descent step with a step that minimizes the UCB in the GIBO algorithm. We show that MinUCB will also converge to local optima with a similar convergence rate as GIBO.  This discovery is also meaningful as it opens up possibilities for new designs on local Bayesian optimization algorithms. In this work we further apply the look ahead strategy to construct the local exploration acquisition function that is more compatible with minimizing the UCB, and propose our second algorithm, Look Ahead Bayesian Optimization through Minimizing UCB (LA-MinUCB). This algorithm is shown to be one step Bayesian optimal, and address our second question as a local efficient BO acquisition function does not require the additional step and accuracy of a gradient estimate. We implement extensive experiments to demonstrate the performance of our algorithms under different settings. The experimental results illustrate that our algorithms have better performance compared to other methods across many synthetic and real-world functions. We summarize our contributions as follows:
\begin{itemize}[leftmargin=0.5cm]
    \item We develop the relationship between the gradient descent step and minimizing the UCB, and show that minimizing UCB is more efficient when the Gaussian process is the underlying surrogate.
    \item We show that minimizing UCB is an efficient and accurate objective for local exploitation and propose MinUCB.
    \item We improve the local exploration acquisition function of MinUCB and obtain a more efficient local Bayesian optimization algorithm LA-MinUCB.
    \item  We apply different synthetic and real-world function on our algorithm, and the results show the effectiveness of our methods.
\end{itemize}

\begin{figure}[tbp]\label{motivation_fig}
    \centering
    \includegraphics[width=6cm]{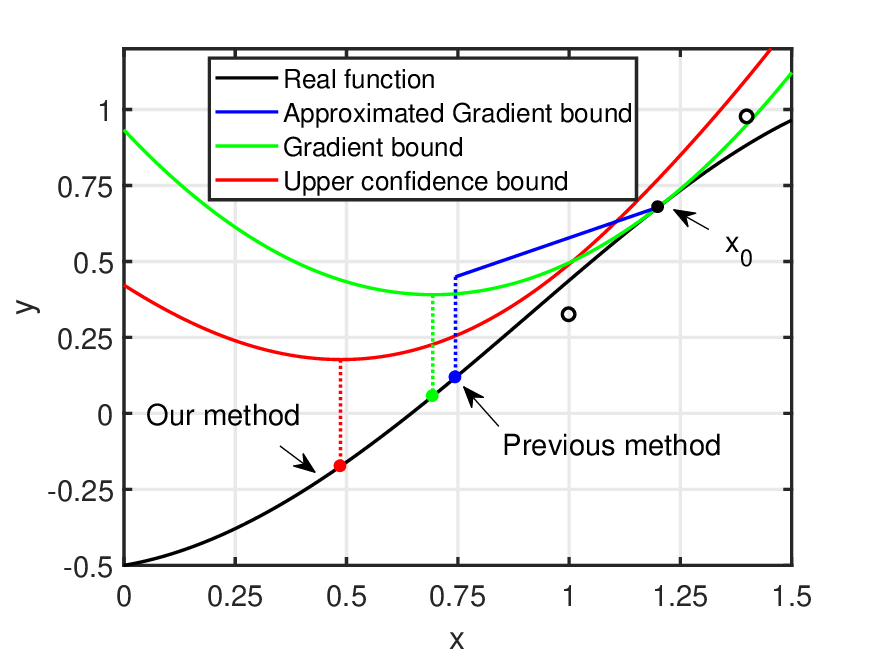}
    \includegraphics[width=6cm]{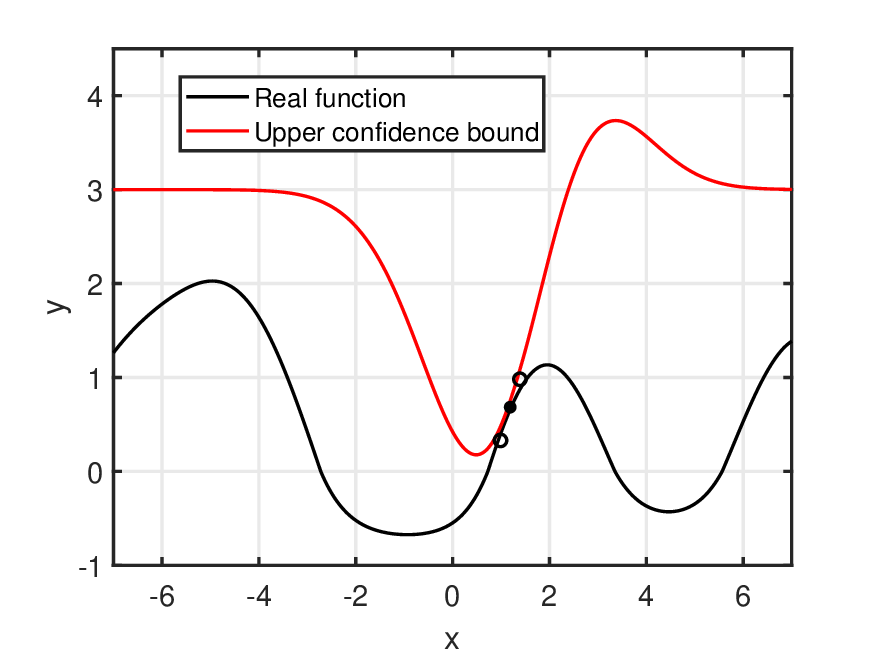}
    \caption{This function $f$ is sampled from $GP(0,k(x,x^{'}))$, where $k(x,x^{'})=exp(-\frac{1}{4}(x-x^{'})^2)$, with standard derivation of white noise $\sigma=0.05$. The dataset contains 2 points, which is marked as black hollow circle. We attempt to search the next point from $\mathbf{x}_0$. The left figure shows that UCB bound is much tighter than other two gradient based bounds, and the minimum points of UCB has the best performance. This shows that minimizing UCB in this example can achieve a much better move to lower point than the gradient descent approach. The right figure illustrates UCB across the design space. Here we see that it is small only near the sampled point, and increases as it moves further away, indicating that minimizing UCB can be viewed as local strategy.}
    \end{figure}     

\section{Literature Review}
High dimensional optimization is a growing research area, where many different methods have been proposed to solve this problem, including Bayesian optimization methods. Among the BO methods, the most widely studied approach focuses on structured Gaussian processes, which imposes additional assumptions, such as low dimensional active spaces or additive spaces, on kernel function or data structure. This includes the Additive and ANOVA Models \cite{binois2022survey}, in which the kernel of Gaussian process is defined through the summation of univariate kernels. This essentially decompose the original Gaussian process into a sum of low dimensional Gaussian processes, which facilitates computations and convergence \cite{kandasamy2015high,li2016high,gardner2017discovering}. Another approach along these lines assumes that the objective function depends only on low dimensional subspace, and examples include hyperparameter optimization for neural networks \cite{bergstra2012random}. Most of these methods suppose a subspace is a linear subspace of the original Euclidean space. Wang et al.\cite{wang2016bayesian} apply
random linear embeddings to approximate this subspace and propose REMBO. Letham et al.  \cite{letham2020re} improve on this and tries to fix over exploration of boundary and distortions in embedding through adaptive linear embeddings \cite{letham2020re}. Other methods to learn the subspace structure include low-dimension matrix recovery \cite{djolonga2013high}, hash
functions construction \cite{nayebi2019framework} and nonlinear embeddings \cite{moriconi2020high}.

As mentioned in the Section \ref{sec:introduction}, another line of research to address the computational and high dimensional challenges is to compromise and focus the BO to be more local, limiting the search region to facilitate computational feasibility and efficiency. Typical approaches along these lines look at incorporating the information about the local optimum \cite{akrour2017local}, using trust regions \cite{eriksson2019scalable,wan2021think,maus2022local}, which expands and shrinks the size of its trust regions with information in each iteration, and approximated gradient methods \cite{muller2021local,nguyen2022local,wu2023behavior}. Among them, the gradient based method MPD \cite{nguyen2022local}, proposed by Nguyen et al. has demonstrated very strong performance compared with other local methods, including gradient based method GIBO \cite{muller2021local}, trust region method Turbo \cite{eriksson2019scalable} and Augmented Random Search (ARS) \cite{mania2018simple}. Given the promising performance of approximate gradient methods, we leverage and extend the works of GIBO and MPD, to provide a simplier and more general local BO approach that can perform as well.

\section{Preliminaries}
\subsection{Bayesian Optimization}
In this paper, we focus on the problem of minimizing a black-box function $f(\cdot)$:
\begin{equation*}
    \min_{\mathbf{x}\in\mathbb{X}}f(\mathbf{x})
\end{equation*}
We assume that no higher order information can be obtained from the oracle (Zeroth-Order-Oracle), where only $i.i.d$ noisy function evaluations $y=f(\mathbf{x})+\varepsilon$, $\varepsilon\sim\mathcal{N}(0,\sigma^2)$ can be observed. Here we focus on applying Bayesian optimization to solve the problem. BO is a surrogate based optimization approach where the objective function is typically modeled with a Gaussian process, and an acquisition function is constructed to sequentially determine next evaluation points in the optimization process.

\subsection{Gaussian process and its derivatives}
 Currently Gaussian process (GP) is one of the most widely used surrogate model today as it has nice analytical form and is flexible to capture various functional forms. A $GP(m(\mathbf{x}),k(\mathbf{x},\mathbf{x}^{'}))$ is specified by its mean function $m(\cdot)$ and kernel function $k(\cdot,\cdot)$. Without loss of generality, we assume the mean function $m(\mathbf{x})\equiv 0$. Suppose $f$ is sampled from the $GP(0,k(\mathbf{x},\mathbf{x}^{'}))$, and we already have a dataset $\mathcal{D}=\{(\mathbf{x}_1,y_1),...,(\mathbf{x}_n,y_n)\}$, and set $\mathbf{X}_{\mathcal{D}}=[\mathbf{x}_1,...,\mathbf{x}_n]$, $\mathbf{y}_{\mathcal{D}}=[y_1,...,y_n]^{T}$, then the posterior over $f$ is also a $GP(\mu_{\mathcal{D}}(\mathbf{x}),k_{\mathcal{D}}(\mathbf{x},\mathbf{x}^{'}))$, where
 \begin{align*}
     \mu_{\mathcal{D}}(\mathbf{x})&=k(\mathbf{x},\mathbf{X}_{\mathcal{D}})(k(\mathbf{X}_{\mathcal{D}},\mathbf{X}_{\mathcal{D}})+\sigma^2I)^{-1}\mathbf{y}_{\mathcal{D}}
     \\
     k_{\mathcal{D}}(\mathbf{x},\mathbf{x}^{'})&=k(\mathbf{x},\mathbf{x}^{'})-k(\mathbf{x},\mathbf{X}_{\mathcal{D}})(k(\mathbf{X}_{\mathcal{D}},\mathbf{X}_{\mathcal{D}})+\sigma^2I)^{-1}k(\mathbf{X}_{\mathcal{D}},\mathbf{x}^{'})
     \\
     \sigma_{\mathcal{D}}^2(\mathbf{x})&=k_{\mathcal{D}}(\mathbf{x},\mathbf{x})
 \end{align*}
 where $k(\mathbf{x},\mathbf{X}_{\mathcal{D}})=[k(\mathbf{x},\mathbf{x}_1),...,k(\mathbf{x},\mathbf{x}_n)]$, $k(\mathbf{X}_{\mathcal{D}},\mathbf{x})=k(\mathbf{x},\mathbf{X}_{\mathcal{D}})^{T}$, and $k(\mathbf{X}_{\mathcal{D}},\mathbf{X}_{\mathcal{D}})$ is the positive definite kernel matrix $[k(\mathbf{x}_{i},\mathbf{x}_{j})]_{i,j\le n}$. 

 It should be noted that if the kernel function $k$ is differentiable, then the derivative of $f$ is also a GP. Given the dataset $\mathcal{D}$, The posterior of $f$ satisfies:
 \begin{equation*}
     \nabla f|\mathcal{D} \sim GP(\nabla \mu_{\mathcal{D}}(\cdot), \nabla  k_{\mathcal{D}}(\cdot,\cdot)\nabla^{T})
 \end{equation*}

In this work we only consider the noisy case, i.e. the standard derivation of white noise  $\sigma>0$.

 \section{The relationship between Gradient descent and minimizing UCB}\label{relationship}
 We first review the traditional Gradient descent approach. This approach is fundamentally based on the assumption of the smoothness of the function:
 \theoremstyle{plain} 
\newtheorem{definition_main}{Definition}
\begin{definition_main}{(Smoothness)}
A function $f$ is L-smooth if and only if for all $\mathbf{x}_1, \mathbf{x}_2\in\mathbb{X}$, we have
\begin{equation*}
    \|\nabla f(\mathbf{x}_1)-\nabla f(\mathbf{x}_2)\|_2\le L\|\mathbf{x}_1-\mathbf{x}_2\|_2
\end{equation*}
\end{definition_main}
 
Suppose the initial point is $\mathbf{x}_{0}$, and function $f(\cdot)$ is $L-smooth$, then we will have the following inequality:
 \begin{equation}\label{gradient_bound}
     f(\mathbf{x})\le f(\mathbf{x}_{0})+<\nabla f(\mathbf{x}_{0}),\mathbf{x}-\mathbf{x}_0> + \frac{L}{2}\|\mathbf{x}-\mathbf{x}_0\|_2^2
 \end{equation}
This provides a quadratic upper bound on $f$, and the minimum value of this upper bound is taken at $\mathbf{x}=\mathbf{x}_{0}-\eta\nabla f(\mathbf{x}_{0})$. In the gradient descent approach, the minimum value of this upper bound is taken as the descent step, where $\eta=\frac{1}{L}$ is used as the step size. In this view Gradient descent can be treated as selecting the minimum point of this quadratic upper bound Eq. (\ref{gradient_bound}).

As direct gradient information is not observable in practice, the upper bound Eq. (\ref{gradient_bound}) cannot be obtained, and approximate gradient methods instead attempt to derive a looser upper bound based on Eq. (\ref{gradient_bound}). In GIBO \cite{muller2021local,wu2023behavior}, they replace the gradient $\nabla f(\mathbf{x}_{0})$ with the derivative of the Gaussian process $\nabla \mu_{\mathcal{D}}(\mathbf{x}_{0})$, and apply the strategy $\mathbf{x}=\mathbf{x}_0-\eta\nabla \mu_{\mathcal{D}}(\mathbf{x}_{0})$ when the variance of this gradient is small enough. According to the proof of Lemma 15 in \cite{wu2023behavior}, GIBO is also a result of optimizing a different upper bound of $f(\cdot)$. If $\mathbf{x}$ is chosen to be the form of $\mathbf{x}=\mathbf{x}_0-\eta\nabla \mu_{\mathcal{D}}(\mathbf{x}_{0})$ and $\eta \le \frac{1}{L}$, this upper bound is:
 \begin{equation}\label{approx_gradient_bound}
     f(\mathbf{x})\le f(\mathbf{x}_{0})-\frac{1}{2}\eta\|\nabla f(\mathbf{x}_{0})\|_2^2 + \frac{1}{2}\eta\|\nabla \mu_{\mathcal{D}}(\mathbf{x}_{0})-\nabla f(\mathbf{x}_{0})\|_2^2
 \end{equation}
If the approximation error of gradient $\|\nabla \mu_{\mathcal{D}}(\mathbf{x}_{0})-\nabla f(\mathbf{x}_{0})\|_2$ is small enough , then the optimal $\eta$ is chosen as $\frac{1}{L}$. This upper bound Eq.(\ref{approx_gradient_bound}) is actually the local exploitation acquisition function of GIBO.

Although these two bounds have intuitive application in the gradient descent, they do have some limitations. The first is the obtainment of the $L-smooth$ coefficient. Although it is possible to estimate the $L-smooth$ coefficient through the Gaussian process, this estimation is expensive as it needs many samples, especially in high dimensional cases. The second is that these two bounds are relatively loose, and the minimum points of these two bounds tend to be too close to $\mathbf{x}_{0}$. Upper bound Eq. (\ref{gradient_bound}) is quadratic and increase very fast when the point $\mathbf{x}$ is far from $\mathbf{x}_0$, and upper bound Eq. (\ref{approx_gradient_bound}) only allows the stepsize $\eta$ to be less than $\frac{1}{L}$. When $L$ is unknown and we have to give it a large estimate to ensure the convergence (usually in real case), the above phenomenon becomes more severe. Taking this view of gradient descent approaches as moving along the minimum of an upper bound of $f(\cdot)$, then leads us to explore if it is possible to discover some tighter upper bounds, where the minimum point is lower than that in Eq. (\ref{gradient_bound}) and Eq. (\ref{approx_gradient_bound}). This can lead us to find a point with a possible higher reward.

A commonly used concept in Bayesian Optimzation is the upper confiedence bound (UCB), which is defined as followed:
\begin{equation}\label{UCB bound}
    UCB(\mathbf{x})=\mu_{\mathcal{D}}(\mathbf{x})+\beta \sigma_{\mathcal{D}}(\mathbf{x})
\end{equation}
Previous work mainly focused on maximizing UCB to find the maximum value of a function \cite{srinivas2009gaussian}. However, it should be noted that UCB is also a natural bound for function $f(\cdot)$. UCB fully utilizes the posterior distribution of $f(\cdot)$, and give every point a probabilistic bound depending on the coefficient $\beta$. The standard deviation term $\sigma_{\mathcal{D}}(\mathbf{x})$ has an upper bound and will not grow faster than the quadratic function, which means the UCB will not change drastically.  This indicates that if we select $\mathbf{x}$ to be
\begin{equation}
    \mathbf{x}^*=\arg\min_{\mathbf{x}}\mu_{\mathcal{D}}(\mathbf{x})+\beta \sigma_{\mathcal{D}}(\mathbf{x})
\end{equation}
Then the function value $f(\mathbf{x}^{*})$ is more likely to be lower than the points obtained through optimizing the upper bound Eq. (\ref{gradient_bound}) and Eq. (\ref{approx_gradient_bound}). This is because the points obtained by minimizing UCB can be further away from the initial point compared to simply gradient descent.

Fig (\ref{motivation_fig}) shows a simple 1-dimension illustrative example. In this example, we sample a function $f$ from $GP(0,k(x,x^{'}))$, where $k(x,x^{'})=exp(-\frac{1}{4}(x-x^{'})^2)$, and the standard deviation of white noise $\sigma=0.05$. Here we illustrate the search of the next point from $x_0$ based on the upper bound perspective of three methods. Suppose we have already sampled two points, which are marked as black hollow circles. These two data points are selected through Central Finite Difference Approximations \cite{shyy1985study}, which is aimed to better estimate the gradient of $\nabla f(\mathbf{x}_0)$. In the left figure, the green line represent the real quadratic upper bound Eq. (\ref{gradient_bound}) at $\mathbf{x}_0$, and the blue line is calculated through Eq. (\ref{approx_gradient_bound}). The coefficient $\beta$ in UCB bound Eq. (\ref{UCB bound}) is set as 3, which means that for any point $\mathbf{x}$, the probability of $f (\mathbf{x})<UCB (\mathbf{x})$ is close to $99.9\%$. It can be seen from the left figure in Fig (\ref{motivation_fig}) that the UCB bound is much tighter than other two bounds, and the minimum point of the UCB bound has a much lower function value than the gradient based method. The right figure plots the UCB across the input space, and we see that the UCB changes relatively slowly and will not reach infinity. Further we observe that UCB is only small near the sampled point, indicating that minimizing the UCB can be view as a local strategy. This simple example illustrates that with a Gaussian process function, minimizing UCB can achieve a better point than the gradient methods, as UCB efficiently utilizes the information from Gaussian process. Based on these insights, we propose two new local Bayesian optimization algorithms, and demonstrate their performances in several numerical examples.

\begin{algorithm}[t]
\SetAlgoLined
\caption{MinUCB: Local Bayesian Optimization through Minimizing UCB}
\label{alg:MinUCB}
\KwData{ A black-box function $f$, and initial point $\mathbf{x}_1$}\\
\For{$t=1,2,...,T$}{
$\mathbf{X}_1=[\mathbf{x}_{t}^{T},...,\mathbf{x}_{t}^{T}]^{T} \in\mathbb{R}^{b^{(1)}_{t}\times d}${\hfill \textit{\#Resample multiple times on $\mathbf{x}_{t}$}}\\
$\mathbf{X}_2=\arg\min_{\mathbf{Z}}\alpha_{trace}(\mathbf{x}_{t},\mathbf{Z})$ where $\mathbf{Z}\in\mathbb{R}^{b^{(2)}_{t}\times d}$ {\hfill \textit{\#Local exploration (sampling)}}\\
$\mathbf{X}=[\mathbf{X}_1^{T},\mathbf{X}_2^{T}]^{T}$  \\  
evaluate the black-box function $f$ on $\mathbf{X}$, obtaining noisy measurements $\mathbf{y}$\\
$\mathcal{D}_{t}=\mathcal{D}_{t-1}\cup(\mathbf{X},\mathbf{y})$\\
$\mathbf{x}_{t+1}=\arg\min_{\mathbf{x}}\mu_{\mathcal{D}_{t}}(\mathbf{x})+\beta_{t} \sigma_{\mathcal{D}_{t}}(\mathbf{x})$ {\hfill \textit{\#Local exploitation (step move)}}\\

}

\end{algorithm}

\section{Local Bayesian Optimization through Minimizing UCB}\label{sec:MinUCB}
The analysis the above section provides us with an important idea, that if we replace the gradient descent step with a step that minimizes a tighter upper bound such as UCB, we may be able to achieve a better result in local optimization. Our first algorithm, Local Bayesian Optimization through Minimizing UCB (MinUCB) (Algorithm \ref{alg:MinUCB}), is developed with this idea, and we show that minimizing the UCB for the step move is an efficient objective for the local exploitation, that can guarantee the convergence with an appropriate local exploration acquisition function.

MinUCB can be viewed as a modified version of GIBO \cite{muller2021local,wu2023behavior} (we list GIBO algorithm in Appendix \ref{sec:details} for reference). In our approach, we adopt the same local exploration acquisition function to sample points as in GIBO (to keep that constant) (line 4 in Algorithm \ref{alg:MinUCB}), and only set the objective of local exploitation acquisition function that drives the step move to minimizing the UCB, instead of gradient descent step $\mathbf{x}_{t+1}=\mathbf{x}_{t}-\eta_{t}\nabla\mu_{\mathcal{D}_{t}}(\mathbf{x}_{t})$, as shown in line 8 in Algorithm \ref{alg:MinUCB}. We first introduce some notations here for better illustration of our algorithm. We define $k_{\mathcal{D}\cup \mathbf{Z}}(\mathbf{x}_{t},\mathbf{x}_{t})=k_{\mathcal{D}}(\mathbf{x}_{t},\mathbf{x}_{t})-k_{\mathcal{D}}(\mathbf{x}_{t},\mathbf{Z})(k_{\mathcal{D}}(\mathbf{Z},\mathbf{Z})+\sigma^2I)^{-1}k_{\mathcal{D}}(\mathbf{Z},\mathbf{x}_{t})$, which is exactly the posterior variance of $f(\mathbf{x}_{t})$ conditioned on the dataset $\mathcal{D}$ and a new input $ \mathbf{Z}$. Because the estimation of variance does not require $\mathbf{y}$, we have omitted the symbol here. The local exploration acquisition function for sampling is defined on this posterior variance:
\begin{equation}
    \alpha_{trace}(\mathbf{x}_{t},\mathbf{Z})=\textbf{tr}(\nabla k_{\mathcal{D}_{t-1}\cup \mathbf{Z}}(\mathbf{x}_{t},\mathbf{x}_{t}) \nabla^{T})
\end{equation}
 which is the trace of the posterior covariance matrix of the $\nabla f(\mathbf{x}_{t})$ conditioned on the dataset and input. This trace quantifies the uncertainty of gradient $\nabla f(\mathbf{x}_{t})$. With a large batch size $b_{t}^{(2)}$, minimizing this trace will result in a lower uncertainty on the estimation of gradient. Although UCB doesn't involve gradient descent, we keep this step constant and argue that the candidates selected through this local exploration acquisition function will still efficiently decrease the uncertainty on this local area, which will benefit the local exploitation move when minimizing the UCB.

In the local exploitation part, the $\beta_t$ controls the search area for each step. The larger the $\beta_t$, the closer $\mathbf{x}_{t+1}$ will be to the existing data point. Minimizing the UCB can bring performance improvements, and we show later in Section \ref{sec:convergence} (Theorem \ref{convergence_speed}) that MinUCB will have a similar convergence rate as GIBO with carefully selection on coefficients $\beta_t$ and batch size $b^{(1)}_{t}$, $b^{(2)}_{t}$. The results of MinUCB can provide inspiration on the design for more efficient local Bayesian optimization algorithm, as shown in Section \ref{sec:LA-MinUCB}.
 
It should be noted that in MinUCB we resample multiple times on the local exploitation result $\mathbf{x}_t$ with a batch size $b_{t}^{(1)}$ (line 3 in Algorithm \ref{alg:MinUCB}). This step is mainly added to ensure the theoretical convergence of the algorithm. This step will typically be only a very small proportion of sampling points, especially in the high dimensional case.

\theoremstyle{plain} 
\newtheorem{theorem_main}{Theorem}
\setcounter{theorem_main}{0}
\newtheorem{assumption}{Assumption}

\section{Convergence Analysis of MinUCB}\label{sec:convergence}
In this section, we establish a convergence analysis of MinUCB to demonstrate the effectiveness of using minimizing UCB as the objective of local exploitation. We prove that MinUCB has a polynomial convergence rate, and this rate also exhibits a polynomial relationship with the input dimension, indicating that MinUCB performs very well in high-dimensional case. For the whole convergence proof of MinUCB, please refer to Appendix B.1-B.4.  In our convergence analysis, we set a mild assumption on kernel function $k$:
\begin{assumption}
    The kernel $k(\cdot,\cdot)$ is stationary, four times continuously differentiable, strictly positive definite, and bounded: $\max_{\mathbf{x}\in\mathbb{X}}k(\mathbf{x},\mathbf{x})\le 1$
\end{assumption}

Many common kernels such as RBF kernel and Matérn kernel with $\gamma>2$ will satisfy this assumption. We also need the definition domain of Gaussian process is bounded:
\begin{assumption}
    The Gaussian process $f(\mathbf{x})$ is defined on a bounded closed set $\mathbb{X}$, i.e. there exist a constant $r>0$ that $\forall \mathbf{x}_1,\mathbf{x}_2\in\mathbb{X}$, $\|\mathbf{x}_1-\mathbf{x}_2\|_2\le r$.
\end{assumption}


Based on the above assumptions, we develop the convergence theory for MinUCB:


\begin{theorem_main}\label{convergence_speed}
     Suppose $f$ is sampled from a zero mean Gaussian process with a continuously differentiable convariance function $k(\cdot,\cdot)$, then if the kernel is RBF kernel or Matérn kernel with $\gamma=2.5$, and satisfy $\beta_{t}=\sqrt{2\log\frac{\pi^2t^2}{\delta}}$ , batch size
		\begin{equation*}
                  b_{t}^{(1)}=
			\begin{cases}
				\log^{2}t\\
				t\\
				t^2
			\end{cases}
            \qquad \text{and} \qquad
			b_{t}^{(2)}=
			\begin{cases}
				d\log^{2}t\\
				dt\\
				dt^2
			\end{cases}
		\end{equation*}
		 Then MinUCB will achieve the convergence rate of
		\begin{equation*}
			\min_{T/2\le i\le T}\|\nabla f(x_i)\|_2^2\le
			\begin{cases}
				O(\sigma d^{\frac{3}{2}}T^{-1}\log^\frac{3}{2} \frac{d^2T^2}{\delta})+O(\sigma d^{2}) \\
				O(\sigma d^{2} T^{-\frac{1}{2}}\log^\frac{5}{2} \frac{d^2T^2}{\delta})=O(\sigma d^{\frac{9}{4}}n^{-\frac{1}{4}}\log^{\frac{5}{2}} \frac{dn}{\delta})\\
				O(\sigma d^{2} T^{-1}\log^\frac{5}{2} \frac{d^2T^2}{\delta}))=O(\sigma d^{\frac{7}{3}}n^{-\frac{1}{3}}\log^{\frac{5}{2}} \frac{dn}{\delta}))
			\end{cases}
		\end{equation*}
\end{theorem_main}
Our method achieves similar results to Wu et al. \cite{wu2023behavior}, except for some logarithmic term and the order of $d$. The dimension $d$ in our boundary is larger than that in Wu's work, which is because we also take the upper bound of the L-smooth coefficient of Gaussian process into consideration, while this upper bound also increases at the polynomial rate with the data dimension $d$, as seen in Theorem \ref{GP_smooth} in Appendix B.1. According to Theorem \ref{convergence_speed}, we need to iteratively increase the UCB coefficient $\beta_{t}$ and  batch size $b_t^{(1)}$, $b_t^{(2)}$ to guarantee the convergence of MinUCB. This phenomenon can be explained that, when the algorithm approaches the local optima, the area around local optima will usually be flatter than other areas (as the gradient is near 0). The algorithm needs more detailed local exploration to ensure that it can descent to a better point. The polynomial convergence rate demonstrates that our local exploitation strategy, minimizing UCB, is accurate and powerful. It can ensure the accuracy of the local search and fully utilize the information of Gaussian processes.

 \section{Look Ahead Bayesian Optimization through Minimizing UCB}\label{sec:LA-MinUCB}
Our proposed MinUCB enjoy good theoretical properties and provide an alternative idea of minimizing the UCB as a good way to progress the local search under a Gaussian process surrogate. However, there are still improvements that can be made to the local exploration in MinUCB. Specifically, UCB itself does not require any gradient information, and the local exploration in MinUCB still focuses on learning the information at a single current point. This under utilizes the Gaussian process surrogate, and the potential information in the local region. Based on the above, we focus here on selecting a better local exploration acquisition function for minimizing UCB, which can help to accelerate the local Bayesian optimization.

In this section we apply a look ahead strategy. The motivation here is to obtain desired candidates to improve the UCB bound, and help the next local exploitation achieve better results under an expectation view. With this idea we propose our second algorithm, Look Ahead Bayesian Optimization through Minimizing UCB (LA-MinUCB) (Algorithm \ref{alg:LA-MinUCB}).

\begin{algorithm}[t]
\SetAlgoLined
\caption{LA-MinUCB: Look Ahead Bayesian Optimization through Minimizing UCB}
\label{alg:LA-MinUCB}
\KwData{ A black-box function $f$.}\\
\For{$t=1,2,...,T$}{
$\mathbf{X}=\arg\min_{\mathbf{Z}}\mathbb{E}_{\mathbf{y_{Z}}}\min_{\mathbf{x}}UCB(\mathbf{x},\mathcal{D}_{t-1},\mathbf{Z},\mathbf{y_{Z}})$ where $\mathbf{Z}\in\mathbb{R}^{b_{t}\times d}${\hfill \textit{\#Local exploration}}\\
evaluate the black-box function $f$ on $\mathbf{X}$, obtaining noisy measurements $\mathbf{y}$\\
$\mathcal{D}_{t}=\mathcal{D}_{t-1}\cup(\mathbf{X},\mathbf{y})$\\
$\mathbf{x}_{t+1}=\arg\min_{\mathbf{x}}\mu_{\mathcal{D}_{t}}(\mathbf{x})+\beta_{t} \sigma_{\mathcal{D}_{t}}(\mathbf{x})${\hfill \textit{\#Local exploitation via minimizing UCB}}\\
$y_{t+1}=f(\mathbf{x}_{t+1})+\varepsilon_{t}$\\
$\mathcal{D}_{t}=\mathcal{D}_{t}\cup(\mathbf{x}_{t+1},y_{t+1})$
}
$\mathbf{x}_{T}=\arg\min_{\mathbf{x}}\mu_{\mathcal{D}_{T}}(\mathbf{x})+\beta_{T} \sigma_{\mathcal{D}_{T}}(\mathbf{x})$
\end{algorithm}

Suppose we have the input $\mathbf{Z}$ and their labels $\mathbf{y}_Z$ (this part is unknown before sampled), we define $\mathcal{D}_Z$ as $\mathcal{D}_Z=\{(\mathbf{z}_{i},y_{i})\}$, $i=1,...,b_t$ as the dataset formed through $\mathbf{Z}$ and $\mathbf{y}_Z$. Then we define
\begin{equation*}
    UCB(\mathbf{x},\mathcal{D}_{t-1},\mathbf{Z},\mathbf{y_{Z}})=\mu_{\mathcal{D}_{t-1}\cup\mathcal{D}_Z}(\mathbf{x})+\beta_{t} \sigma_{\mathcal{D}_{t-1}\cup \mathcal{D}_Z}(\mathbf{x})
\end{equation*}
This is the upper confidence bound when we already have the dataset $\mathcal{D}_{t-1}\cup \mathcal{D}_Z$, and we want to find the input $\mathbf{Z}$ to minimize the minimum point of UCB: $\min_{\mathbf{x}}UCB(\mathbf{x},\mathcal{D}_{t-1},\mathbf{Z},\mathbf{y_{Z}})$. However, as the label $\mathbf{y}_Z$ is unknown, we can only choose to optimize it through its expectation. We adopt this look ahead predictive as the local exploration acquisition function in LA-MinUCB:
\begin{equation}
    \mathbf{X}=\arg\min_{\mathbf{Z}}\mathbb{E}_{\mathbf{y_{Z}}}\min_{\mathbf{x}}UCB(\mathbf{x},\mathcal{D}_{t-1},\mathbf{Z},\mathbf{y_{Z}})
\end{equation}

Although the local exploration in LA-MinUCB does not need to be specified around a certain point, it is still necessary to have an local exploitation step. Local exploitation step may find a point with current best reward, and will provide a better foundation for subsequent local exploration.

LA-MinUCB has a similar structure with traditional Knowledge Gradient \cite{frazier2018tutorial} except for the standard derivation term $\sigma_{\mathcal{D}_{t}}(\mathbf{x})$. This standard derivation term behaved as a regularization term, that force the sampled points to be not too far away from current area. This is because only when the selected points are closer to the current optimal point, are they more likely to learn which nearby area may contain smaller values, thereby achieving a greater decrease in the next exploitation step. This extra standard derivation term restrict the over-exploration of the original Knowledge Gradient to help it perform better exploitation. Benefiting from this property, LA-MinUCB is able to focus on doing local search and quickly find local optima points.

The significant advantage of using LA-MinUCB is that its strategy has good theoretical properties:
\begin{theorem_main}
    If only one iteration is left and we can observe the function value through sampling, then the local exploration in LA-MinUCB is Bayes-optimal among all feasible policies.
\end{theorem_main}
The proof of this theorem is listed in Appendix B.5. This indicates that LA-MinUCB is a greedy strategy, where in each step it use the optimal acquisition function to enhance the exploitation. From this perspective, LA-MinUCB is likely to be superior to MinUCB because LA-MinUCB is better prepared for the next step of exploitation. This is also reflected in our experiments results, which shows that LA-MinUCB have a very competitive performance in numerical experiments and practical applications. Meanwhile, this also illustrates that good local Bayesian optimization algorithms can be constructed without the need for approximate gradients.

\begin{figure}[tbp]

    \hspace{-1cm}\includegraphics[width=16cm]{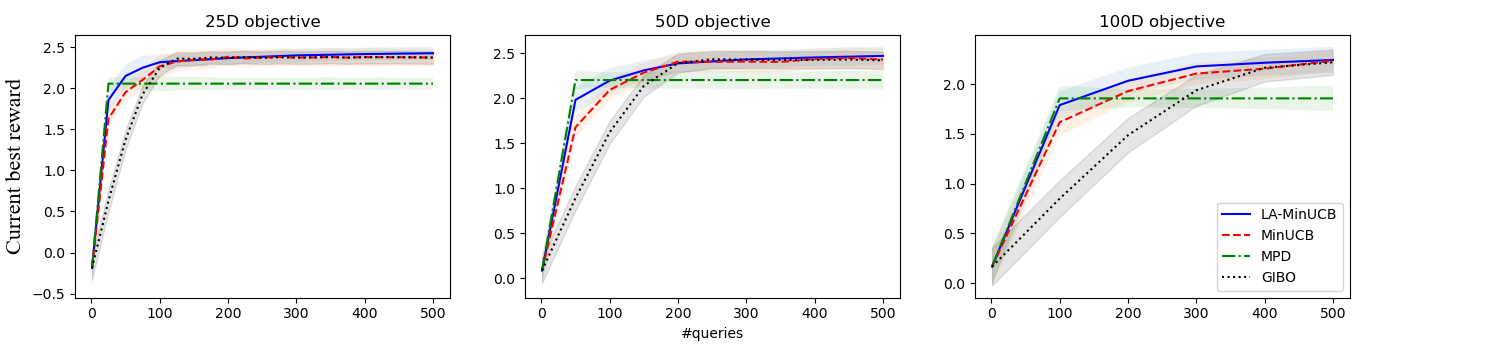}
    \caption{Progressive optimized reward on high-dimensional synthetic functions. LA-MinUCB demonstrates fast and accurate convergence compared to other methods.}
    \label{experimental_1}
\end{figure}     

\begin{figure}[tbp]
    \hspace{-1cm}
    \includegraphics[width=16cm]{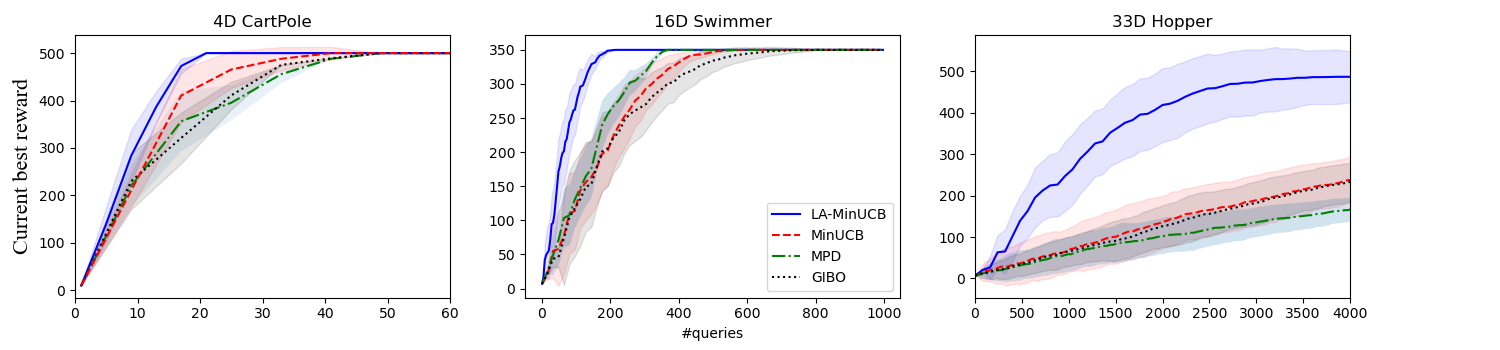}
    \caption{Progressive optimized reward on the MuJuCo tasks. LA-MinUCB has consistently optimal performance.}
    \label{experimental_2}
\end{figure}   

\section{Experiments}\label{experiment}
In this section we apply varies experimental settings to assess the efficacy of our proposed two algorithms, Min-UCB and LA-MinUCB, relative to two established methodologies. These include two approximated gradient methods, (1) GIBO \cite{muller2021local} and MPD \cite{nguyen2022local}, which has shown strong performance in local Bayesian optimization methods. Our code is based on Nguyen et al. \cite{nguyen2022local}, where they provide the program of GIBO, MPD, and various of objective functions. Our experimental settings on synthetic (Sect. 4.1) and reinforcement learning (Sect. 4.2) objectives are same as their papers. Each algorithm is executed a total of ten times for every objective function that we examine, initiating from an identical set of starting points sampled across the bounded domain via a Sobol sequence. We illustrate the results in Fig. (\ref{experimental_1}) (synthetic objective) and (\ref{experimental_2}) (reinforcement learning objective). The figures shows the mean of current best reward on the number of queries with an error bar (defined through the standard derivation).  Our experimental framework was executed on a workstation of 20 Intel Xeon CPU cores, with a 32GB of memory.
The codes for the proposed algorithm and experiments will also be made publicly available upon publication. For more computational details, please refer to Appendix \ref{sec:details}.

\subsection{Synthetic Objectives}

In our first experiments, we focused on optimizing synthetic objective functions within the d-dimensional unit hypercube $[0,1]^{d}$. These functions were generated by sampling from a Gaussian Process (GP) with a Radial Basis Function (RBF) kernel. This Synthetic objective is first mentioned in M{\"u}ller et al. \cite{muller2021local} , and for a more comprehensive understanding of this experimental setup, please refer to §4.1 in their work.  Each experiment was allocated a budget of 500 function evaluations, under the dimension $\{25,50,100\}$ separately. The error bar here is defined as plus or minus 0.2 standard derivation for better illustration, as the functions sampled by Gaussian processes have significant differences. Fig. (\ref{experimental_1}) illustrates that our proposed LA-MinUCB and Min-UCB achieve highly competitive performance in both efficiency and accuracy, where LA-MinUCB is slightly better. The MPD ascent very rapidly at the beginning, but soon shows instability. Although each ascent applies the direction of maximum probability, it considers multiple ascent in one iteration in the algorithm. This strategy may instead lead to a small probability that the final point is greater than the original point, causing the algorithm to fall into suboptimal solutions. GIBO can achieve very close to the real local optimum, but its strategy is too conservative, resulting in slow convergence.

\subsection{Reinforcement Learning Objective}

In this experiment, we turned to reinforcement learning, specifically MuJoCo-based locomotion tasks \cite{todorov2012mujoco}. In these tasks, we try to learn a linear policy that maps states into actions in order to maximize the rewards given by the learning environment. We adopt the experiments similar with Nguyen et al. \cite{nguyen2022local} including three environments: CartPole-v1 with 4 parameters, Swimmer-v1 with 16 parameters, and Hopper-v1 with 33 parameters. The only difference is that we apply the Hopper-v1 without the state normalization, as in Nguyen's work they show that this state normalization will cause a significant unstable change in the function value. The error bar here is defined as plus or minus one standard derivation. Our results, depicted in Fig. (\ref{experimental_2}), show that LA-MinUCB has an excellent performance in all three cases. LA-MinUCB exhibits faster convergence speed and may achieve a better local optima than other methods. Although MinUCB has also shown slightly better performance than GIBO, it is still not that efficient as LA-MinUCB. This reflects the powerful performance of LA-MinUCB in real applications, and also indicates that efficient local Bayesian optimization algorithm can be designed without the estimation of gradient information.

\section{Conclusion}
In this paper we find the connection between gradient descent and minimizing the UCB, demonstrating that minimizing UCB is an efficient local exploitation strategy with the Gaussian Processes surrogate. We introduce MinUCB, an algorithm applying minimizing the UCB as the objective of local exploitation, that will converges to local optima within a polynomial rate. By enhancing the local exploration acquisition function of MinUCB, we have developed a more advanced local Bayesian optimization algorithm, LA-MinUCB. We have tested our algorithm on various synthetic and reinforcement learning objectives, and the results have confirmed the efficacy of our approaches.

In this article, we have not provided the convergence proof of the LA-MinUCB algorithm. One consideration is whether LA-MinUCB have a better theoretical convergence rate compared to other methods, such as GIBO? Additionally, are there any other better local exploitation strategy in local Bayesian optimization? These will become interesting research directions in the future.

\newpage
{\large\textbf{Acknowledgments and Disclosure of Funding}}

Szu Hui Ng’s work is supported in part by the Ministry of Education, Singapore (grant: R-266-000-149-114).

\bibliography{reference}

\newpage

\appendix
\theoremstyle{plain} 
\newtheorem{theorem}{Theorem}
\setcounter{theorem}{0}
\renewcommand{\thetheorem}{\arabic{theorem}}

\theoremstyle{plain} 
\newtheorem{lemma}{Lemma}
\theoremstyle{plain}

\section{Additional details of the algorithms}\label{sec:details}
In this section we will provide additional details about our proposed algorithms. First we list the GIBO algorithm in the version of Wu et al.\cite{wu2023behavior}, as shown in Algorithm \ref{alg:GIBO}. MinUCB adopt the same local exploration acquisition function to sample points as in GIBO, and change the objective of local exploitation acquisition function into minimizing the UCB, instead of Gradient descent step.

\begin{algorithm}[t]
\SetAlgoLined
\caption{GIBO}
\label{alg:GIBO}
\KwData{ A black-box function $f$}\\
\For{$t=1,2,...,T$}{
$\mathbf{X}=\arg\min_{\mathbf{Z}}\alpha_{trace}(\mathbf{x}_{t},\mathbf{Z})$ where $\mathbf{Z}\in\mathbb{R}^{b_{t}\times d}$ {\hfill \textit{\#Local exploration (sampling)}}\\
evaluate the black-box function $f$ on $\mathbf{X}$, obtaining noisy measurements $\mathbf{y}$\\
$\mathcal{D}_{t}=\mathcal{D}_{t-1}\cup(\mathbf{X},\mathbf{y})$\\
$\mathbf{x}_{t+1}=\mathbf{x}_{t}-\eta_{t}\nabla \mu_{\mathcal{D}_{t}}(\mathbf{x}_{t})$ {\hfill \textit{\#Local exploitation (step move)}}\\

}
\end{algorithm}
As for the computational problem of LA-MinUCB, the optimization for the local exploration acquisition function is quite hard, as it involves a nested optimization problem. To handle this task we use the idea of one-shot optimization in BoTorch \cite{balandat2020botorch}, that transform the original problem into a deterministic optimization problem with fixed sampling. This method is first proposed to solve the optimization of Knowledge Gradient. However, considering that our method is very similar to Knowledge Gradient, their optimization method is also applicable to our problem. For more information about this one-shot optimization method, please refer to the Section 4.2 in Balandat et al. \cite{balandat2020botorch}.

In our numerical experiment, which has been shown in Section \ref{experiment}, we apply the same experimental settings with M{\"u}ller et al. \cite{muller2021local} and Nguyen et al. \cite{nguyen2022local}. The coefficients in GIBO and MPD are chosen as the optimal one, which are reported in their experiments or code. As for our MinUCB and LA-MinUCB, we use a fixed batch size in each experiment, i.e. $b_t^{(1)}$, $b_t^{(2)}$ and $b_t$ are unchanged in each experiment. The UCB coeffieient $\beta^t$ in MinUCB and LA-MinUCB is fixed as 3, which has been shown to be able to balance the convergence speed and accuracy of the algorithm.

\section{Theoretical analysis for MinUCB and LA-MinUCB}
In this section, we will present the theoretical results of the algorithms introduced in the article. Sections B.1-B.4 provide the theoretical convergence proof for MinUCB, while B.5 discusses the theoretical properties of LA-MinUCB. Sections B.1-B.3 serve as the preliminary theory for proving the convergence of MinUCB, mainly presenting the local smoothness properties of Gaussian processes, mean functions, and standard deviation functions, as well as the theoretical properties of approximate gradients. B.4 combines the above theories,  connecting the minimization of UCB with gradient descent, and provides the convergence theorem for MinUCB. B.5 demonstrates that LA-MinUCB has one-step Bayesian optimality, illustrating the superiority of the algorithm.

\subsection{The function smoothness in Gaussian Process and mean function}\label{sec:smooth_gp_mu}
In this subsection we mainly build the theory for the smoothness of Gaussian process and mean function. Our convergence proof heavily depend on the smoothness analysis on the Gaussian process, mean function and standard derivation function. These proof rely on the upper bound of Gaussian processes proposed by Lederer et al.\cite{lederer2019uniform}, which is related to the Lipshitz continues coefficient on kernel function:
\setcounter{theorem}{2}

\begin{lemma}[Lederer et al.\cite{lederer2019uniform}]
    
    Consider a zero mean Gaussian process with a continuously differentiable convariance function $k(\cdot,\cdot)$, and let $L_k$ denote its Lipschitz constant on the set $\mathbb{X}$ with maximum extension $r=\max_{\mathbf{x},\mathbf{x}^{'}\in \mathbb{X}}\|\mathbf{x}-\mathbf{x}^{'}\|_2$. Then, with probability of at least $1-\delta$, the supremum of a sample function $f(\mathbf{x})$ of this Gaussian process is bounded by:
    \begin{equation*}
        \sup_{\mathbf{x}\in\mathbb{X}}f(\mathbf{x})\le\sqrt{2\log\left(\frac{1}{\delta}\right)}\max_{\mathbf{x}\in\mathbb{X}}\sqrt{k(\mathbf{x},\mathbf{x})}+12\sqrt{6d}\max\left\{\max_{\mathbf{x}\in\mathcal{X}}\sqrt{k(\mathbf{x},\mathbf{x})},\sqrt{rL_k}\right\}
    \end{equation*}
\end{lemma}

We use this lemma to build the upper bound on the L-smoothness coefficient on the Gaussian process and mean function:

\begin{theorem}[Smoothness of Gaussian process]\label{GP_smooth}
    Consider a zero mean Gaussian process with a continuously differentiable convariance function $k(\cdot,\cdot)$, and $k(\cdot,\cdot)$ satisfies Assumption 1. We define the second-order partial derivative $k^{\partial_{ij}}(\cdot,\cdot)$ as 
    \begin{equation*}
        k^{\partial_{ij}}(\mathbf{x},\mathbf{x}^{'})=\frac{\partial^{4}}{\partial x_{i}\partial x_{j}\partial x_{i}^{'}\partial x_{j}^{'}}k(\mathbf{x},\mathbf{x}^{'})
    \end{equation*}
    Let $L_k^{\partial_{ij}}$ denote the Lipschitz constants of the second-order partial derivative $k^{\partial_{ij}}(\cdot,\cdot)$ on the set $\mathbb{X}$ with maximum extension $r=\max_{\mathbf{x},\mathbf{x}^{'}\in \mathbb{X}}\|\mathbf{x}-\mathbf{x}^{'}\|_2$. Then, a sample function $f(\cdot)$ is $\beta$-smooth with probability of at least $1-\delta$, and $\beta$ satisfies:
    \begin{equation*}
        \beta\le\sqrt{\sum_{i,j}U_{ij}^2}
    \end{equation*}
    where
    \begin{equation*}
        U_{ij}=\sqrt{2\log\left(\frac{2d^2}{\delta}\right)}\max_{\mathbf{x}\in\mathbb{X}}\sqrt{k^{\partial_{ij}}(\mathbf{x},\mathbf{x})}+12\sqrt{6d}\max\left\{\max_{\mathbf{x}\in\mathbb{X}}\sqrt{k^{\partial_{ij}}(\mathbf{x},\mathbf{x})},\sqrt{rL^{\partial_{ij}}_k}\right\}
    \end{equation*}
\end{theorem}

\begin{proof}
    We set the Hessian matrix of $f(\mathbf{x})$ as 
    \begin{equation*}
        H(\mathbf{x})=
        \begin{bmatrix}
            \frac{\partial^2}{\partial x_1 \partial x_1}f(\mathbf{x})& \cdots& \frac{\partial^2}{\partial x_1 \partial x_d}f(\mathbf{x})\\
            \vdots&\ddots&\vdots\\
            \frac{\partial^2}{\partial x_d \partial x_1}f(\mathbf{x})& \cdots& \frac{\partial^2}{\partial x_d \partial x_d}f(\mathbf{x})
        \end{bmatrix}
    \end{equation*}
    According to the equivalence of matrix norms, we have 
    \begin{equation*}
        \|H(\mathbf{x})\|_2\le \|H(\mathbf{x})\|_F
    \end{equation*}
    Thus 
    \begin{equation}\label{gp smooth 1}
        \beta\le\sup_{\mathbf{x}\in\mathbb{X}}\|H(\mathbf{x})\|_2\le\sup_{\mathbf{x}\in\mathbb{X}}\|H(\mathbf{x})\|_F = \sup_{\mathbf{x}\in\mathbb{X}} \sqrt{\sum_{i,j}^{d} (\frac{\partial^2}{\partial x_i \partial x_j}f(\mathbf{x}))^2}
    \end{equation}
    Then through applying Lemma 2, we have the following equation holds with probability at least $1-\delta/d^2$
    \begin{equation}\label{gp smooth 2}
        \sup_{\mathbf{x}\in\mathbb{X}}\left|\frac{\partial^2}{\partial x_i \partial x_j}f(\mathbf{x})\right|\le \sqrt{2\log\left(\frac{2d^2}{\delta}\right)}\max_{\mathbf{x}\in\mathbb{X}}\sqrt{k^{\partial_{ij}}(\mathbf{x},\mathbf{x})}+12\sqrt{6d}\max\left\{\max_{\mathbf{x}\in\mathbb{X}}\sqrt{k^{\partial_{ij}}(\mathbf{x},\mathbf{x})},\sqrt{rL^{\partial_{ij}}_k}\right\}
    \end{equation}
    Then if we combine the Eq.(\ref{gp smooth 1}) and Eq.(\ref{gp smooth 2}), we will get the result of Theorem 1.
\end{proof}

\begin{theorem} \label{mu_bound}
    If $\sup_{\mathbf{x}\in\mathbb{X}} |f(\mathbf{x})|<L$, then with the probability of at least $1-\delta$, for a dataset $\mathcal{D}$, we have 
        \begin{equation*}
        \max_{\mathbf{x}\in\mathbb{X}} |\mu_{\mathcal{D}}(\mathbf{x})| \le \sqrt{2\log\left(\frac{1}{\delta}\right)}\max_{\mathbf{x}\in\mathbb{X}}\sqrt{k(\mathbf{x},\mathbf{x})}+L
    \end{equation*}
\end{theorem}

\begin{proof}
    For a dataset $\mathcal{D}$, we have the mean function $\mu_{\mathcal{D}}(\mathbf{x})$ and variance function $\sigma_{\mathcal{D}}^2(\mathbf{x})$.
    \begin{equation*}
        \mathbf{x}^{*}=\arg\max_{\mathbf{x}\in\mathbb{X}} |\mu_{\mathcal{D}}(\mathbf{x})|
    \end{equation*}
    \begin{equation}\label{mu_bound_1}
    \begin{aligned}
        |\mu_{\mathcal{D}}(\mathbf{x}^{*})|
        &\le |f(\mathbf{x}^{*})|+|\mu_{\mathcal{D}}(\mathbf{x}^{*})-f(\mathbf{x}^{*})|\\
        &\le \sup_{\mathbf{x}\in\mathbb{X}}|f(\mathbf{x})| + |\mu_{\mathcal{D}}(\mathbf{x}^{*})-f(\mathbf{x}^{*})|
    \end{aligned}
    \end{equation}
    For a random variable $r\sim\mathcal{N}(0,1)$, we have $P(|r|>c)\le \exp(-\frac{c^2}{2})$. Thus with the probability of at least $1-\frac{\delta}{2}$, we have
    \begin{equation} \label{mu_bound_2}
        |\mu_{\mathcal{D}}(\mathbf{x}^{*})-f(\mathbf{x}^{*})|\le \sqrt{2\log\frac{1}{\delta}} \sigma_{\mathcal{D}}(\mathbf{x}^{*}) \le \sqrt{2\log\frac{1}{\delta}} \max_{\mathbf{x}\in\mathbb{X}} \sqrt{k(\mathbf{x},\mathbf{x})}
    \end{equation}

    Combine the result of Eq(\ref{mu_bound_1}) and Eq(\ref{mu_bound_2}), we have the following results with the probability of at least $1-\delta$:
    \begin{equation*}
        \max_{\mathbf{x}\in\mathbb{X}} |\mu_{\mathcal{D}}(\mathbf{x})| \le \sqrt{2\log\left(\frac{1}{\delta}\right)}\max_{\mathbf{x}\in\mathbb{X}}\sqrt{k(\mathbf{x},\mathbf{x})}+L
    \end{equation*}
    
\end{proof}

\begin{theorem}[Smoothness of mean function]\label{mu_L_smooth}
    Consider a zero mean Gaussian process with a continuously differentiable convariance function $k(\cdot,\cdot)$, and $k(\cdot,\cdot)$ satisfies Assumption 1. Let $L_k^{\partial_{ij}}$ denote the Lipschitz constants of the second-order partial derivative $k^{\partial_{ij}}(\cdot,\cdot)$ on the set $\mathbb{X}$ with maximum extension $r=\max_{\mathbf{x},\mathbf{x}^{'}\in \mathbb{X}}\|\mathbf{x}-\mathbf{x}^{'}\|_2$. Then, given a dataset $\mathcal{D}$, the mean function $\mu_{\mathcal{D}}(\mathbf{x})$ is $\beta_{\mu}$-smooth with probability of at least $1-\delta$, and $\beta$ satisfies:
    \begin{equation*}
        \beta_{\mu}\le\sqrt{\sum_{i,j}V_{ij}^2}
    \end{equation*}
    where
    \begin{equation*}
        V_{ij}=\sqrt{2\log\left(\frac{d^2}{\delta}\right)}\max_{\mathbf{x}\in\mathbb{X}}\sqrt{k^{\partial_{ij}}(\mathbf{x},\mathbf{x})}+L_{ij}
    \end{equation*}
    and $L_{ij}=\sup_{\mathbf{x}\in\mathbb{X}}\left|\frac{\partial^2}{\partial x_i \partial x_j}f(\mathbf{x})\right|$
\end{theorem}

\begin{proof}
    We set the Hessian matrix of $\mu_{\mathcal{D}}(\mathbf{x})$ as 
    \begin{equation*}
        H_{\mu}(\mathbf{x})=
        \begin{bmatrix}
            \frac{\partial^2}{\partial x_1 \partial x_1}\mu_{\mathcal{D}}(\mathbf{x})& \cdots& \frac{\partial^2}{\partial x_1 \partial x_d}\mu_{\mathcal{D}}(\mathbf{x})\\
            \vdots&\ddots&\vdots\\
            \frac{\partial^2}{\partial x_d \partial x_1}\mu_{\mathcal{D}}(\mathbf{x})& \cdots& \frac{\partial^2}{\partial x_d \partial x_d}\mu_{\mathcal{D}}(\mathbf{x})
        \end{bmatrix}
    \end{equation*}
    According to the equivalence of matrix norms, we have 
    \begin{equation*}
        \|H_{\mu}(\mathbf{x})\|_2\le \|H_{\mu}(\mathbf{x})\|_F
    \end{equation*}

    Thus we can apply Theorem \ref{GP_smooth} and get the following result probability of at least $1-\delta$:
    
    \begin{equation}
        \beta_{\mu}\le\sup_{\mathbf{x}\in\mathbb{X}}\|H_{\mu}(\mathbf{x})\|_2\le\sup_{\mathbf{x}\in\mathbb{X}}\|H_{\mu}(\mathbf{x})\|_F = \max_{\mathbf{x}\in\mathbb{X}} \sqrt{\sum_{i,j}^{d} (\frac{\partial^2}{\partial x_i \partial x_j}\mu_{\mathcal{D}}(\mathbf{x}))^2}
        \le \sqrt{\sum_{i,j}^{d} \max_{\mathbf{x}\in\mathbb{X}}(\frac{\partial^2}{\partial x_i \partial x_j}\mu_{\mathcal{D}}(\mathbf{x}))^2}
    \end{equation}
    where
    \begin{equation*}
        \max_{\mathbf{x}\in\mathbb{X}}\left|\frac{\partial^2}{\partial x_i \partial x_j}\mu_{\mathcal{D}}(\mathbf{x})\right| \le \sqrt{2\log\left(\frac{d^2}{\delta}\right)}\max_{\mathbf{x}\in\mathbb{X}}\sqrt{k^{\partial_{ij}}(\mathbf{x},\mathbf{x})}+L_{ij}
    \end{equation*}

\end{proof}

Here we need an additional Theorem to give the upper bound on the L-smoothness coefficient on the posterior of $f(\cdot)$, which is aim to bound the probability under a specific L-smoothness coefficient.

\begin{lemma}
    
    Consider a zero mean Gaussian process with a continuously differentiable convariance function $k(\cdot,\cdot)$, and let $L_k$ denote its Lipschitz constant on the set $\mathbb{X}$ with maximum extension $r=\max_{\mathbf{x},\mathbf{x}^{'}\in \mathbb{X}}\|\mathbf{x}-\mathbf{x}^{'}\|_2$. Then, given a dataset $\mathcal{D}$, we define the posterior of $f(\cdot)$ over dataset $\mathcal{D}$ as $f_{\mathcal{D}}(\cdot)$. With probability of at least $1-\delta$, the supremum of a sample function $f_{\mathcal{D}}(\mathbf{x})$ of this Gaussian process is bounded by:
    \begin{equation*}
        \sup_{\mathbf{x}\in\mathbb{X}}f_{\mathcal{D}}(\mathbf{x})\le\sqrt{2\log\left(\frac{1}{\delta}\right)}\max_{\mathbf{x}\in\mathbb{X}}\sqrt{k(\mathbf{x},\mathbf{x})}+12\sqrt{6d}\max\left\{\max_{\mathbf{x}\in\mathcal{X}}\sqrt{k(\mathbf{x},\mathbf{x})},\sqrt{rL_k}\right\}
    \end{equation*}
\end{lemma}

\begin{proof}
     Note that for the covariance pseudo-metric $d_{k}(x,x^{\prime})=\sqrt{k(x,x)+k(x^{\prime},x^{\prime})-2k{(x,x^{\prime})}}$, we have
    \begin{align*}
    k_{\mathcal{D}}(\mathbf{x},\mathbf{x})+k_{\mathcal{D}}(\mathbf{x}^{\prime},\mathbf{x}^{\prime})-2k_{\mathcal{D}}{(\mathbf{x},\mathbf{x}^{\prime})}
    =&
    k(\mathbf{x},\mathbf{x})+k(\mathbf{x}^{\prime},\mathbf{x}^{\prime})-2k{(\mathbf{x},\mathbf{x}^{\prime})}
    \\
    &-(k(\mathbf{x},\mathbf{X}_{\mathcal{D}})-k(\mathbf{x}^{'},\mathbf{X}_{\mathcal{D}}))(k(\mathbf{X}_{\mathcal{D}},\mathbf{X}_{\mathcal{D}})+\sigma^2I)^{-1}(k(\mathbf{x},\mathbf{X}_{\mathcal{D}})-k(\mathbf{x}^{'},\mathbf{X}_{\mathcal{D}}))^{T}
    \\
    \le&
    k(\mathbf{x},\mathbf{x})+k(\mathbf{x}^{\prime},\mathbf{x}^{\prime})-2k{(\mathbf{x},\mathbf{x}^{\prime})}
    \end{align*}
    Then this Lemma can be proved in a same way with Lemma B.1 in Lederer et al. \cite{lederer2019uniform}.
\end{proof}

\begin{theorem}[Smoothness of the posterior Gaussian process]\label{conditional_L_smooth}
    Consider a zero mean Gaussian process with a continuously differentiable convariance function $k(\cdot,\cdot)$, and $k(\cdot,\cdot)$ satisfies Assumption 1. Let $L_k^{\partial_{ij}}$ denote the Lipschitz constants of the second-order partial derivative $k^{\partial_{ij}}(\cdot,\cdot)$ on the set $\mathbb{X}$ with maximum extension $r=\max_{\mathbf{x},\mathbf{x}^{'}\in \mathbb{X}}\|\mathbf{x}-\mathbf{x}^{'}\|_2$.  Then, a sample function $f_{\mathcal{D}}(\cdot)$ is $\beta_{\mathcal{D}}$-smooth with probability of at least $1-\delta$, and $\beta$ satisfies:
    \begin{equation*}
        \beta_{\mathcal{D}}\le\sqrt{\sum_{i,j}\Tilde{U}_{ij}^2}
    \end{equation*}
    where
    \begin{equation*}
        \Tilde{U}_{ij}=L_{ij}+2\sqrt{2\log\left(\frac{4d^2}{\delta}\right)}\max_{\mathbf{x}\in\mathbb{X}}\sqrt{k^{\partial_{ij}}(\mathbf{x},\mathbf{x})}+12\sqrt{6d}\max\left\{\max_{\mathbf{x}\in\mathcal{X}}\sqrt{k^{\partial_{ij}}(\mathbf{x},\mathbf{x})},\sqrt{rL^{\partial_{ij}}_k}\right\}
    \end{equation*}
    and $L_{ij}=\sup_{\mathbf{x}\in\mathbb{X}}\left|\frac{\partial^2}{\partial x_i \partial x_j}f(\mathbf{x})\right|$
\end{theorem}

\begin{proof}
    We set the Hessian matrix of $f_{\mathcal{D}}(\mathbf{x})$ as 
    \begin{equation*}
        H_{\mathcal{D}}(\mathbf{x})=
        \begin{bmatrix}
            \frac{\partial^2}{\partial x_1 \partial x_1}f_{\mathcal{D}}(\mathbf{x})& \cdots& \frac{\partial^2}{\partial x_1 \partial x_d}f_{\mathcal{D}}(\mathbf{x})\\
            \vdots&\ddots&\vdots\\
            \frac{\partial^2}{\partial x_d \partial x_1}f_{\mathcal{D}}(\mathbf{x})& \cdots& \frac{\partial^2}{\partial x_d \partial x_d}f_{\mathcal{D}}(\mathbf{x})
        \end{bmatrix}
    \end{equation*}

    \begin{equation}\label{condition_smooth_1}
        \beta_{\mathcal{D}}\le\sup_{\mathbf{x}\in\mathbb{X}}\|H_{\mathcal{D}}(\mathbf{x})\|_2\le\sup_{\mathbf{x}\in\mathbb{X}}\|H_{\mathcal{D}}(\mathbf{x})\|_F = \sup_{\mathbf{x}\in\mathbb{X}} \sqrt{\sum_{i,j}^{d} (\frac{\partial^2}{\partial x_i \partial x_j}f_{\mathcal{D}}(\mathbf{x}))^2}
        \le \sqrt{\sum_{i,j}^{d} \sup_{\mathbf{x}\in\mathbb{X}}(\frac{\partial^2}{\partial x_i \partial x_j}f_{\mathcal{D}}(\mathbf{x}))^2}
    \end{equation}
    The posterior $f_{\mathcal{D}}(\cdot)$ can be divided into two parts:
    \begin{equation}\label{condition_smooth_2}
        \sup_{\mathbf{x}\in\mathbb{X}}\left|\frac{\partial^2}{\partial x_i \partial x_j}f_{\mathcal{D}}(\mathbf{x})\right| \le  \sup_{\mathbf{x}\in\mathbb{X}}\left|\frac{\partial^2}{\partial x_i \partial x_j}\mu_{\mathcal{D}}(\mathbf{x})\right| + \sup_{\mathbf{x}\in\mathbb{X}}\left|\frac{\partial^2}{\partial x_i \partial x_j}f_{\mathcal{D}}(\mathbf{x})-\frac{\partial^2}{\partial x_i \partial x_j}\mu_{\mathcal{D}}(\mathbf{x})\right|
    \end{equation}
    where $\frac{\partial^2}{\partial x_i \partial x_j}f_{\mathcal{D}}(\mathbf{x})-\frac{\partial^2}{\partial x_i \partial x_j}\mu_{\mathcal{D}}(\mathbf{x})$ is a zero mean Gaussian process with covariance function $k^{\partial ij}_{\mathcal{D}}(\mathbf{x},\mathbf{x}^{'})$. According to Theorem?, we have the following result with probability at least $1-\frac{\delta}{2d^2}$

    \begin{equation}\label{condition_smooth_3}
    \begin{aligned}
        \sup_{\mathbf{x}\in\mathbb{X}}&\left|\frac{\partial^2}{\partial x_i \partial x_j}f_{\mathcal{D}}(\mathbf{x})-\frac{\partial^2}{\partial x_i \partial x_j}\mu_{\mathcal{D}}(\mathbf{x})\right|\\
        &\le \sqrt{2\log\left(\frac{4d^2}{\delta}\right)}\max_{\mathbf{x}\in\mathbb{X}}\sqrt{k^{\partial_{ij}}(\mathbf{x},\mathbf{x})}+12\sqrt{6d}\max\left\{\max_{\mathbf{x}\in\mathcal{X}}\sqrt{k^{\partial_{ij}}(\mathbf{x},\mathbf{x})},\sqrt{rL^{\partial_{ij}}_k}\right\}
    \end{aligned}
    \end{equation}

    Through the result of Lemma \ref{mu_bound}, with the probability of at least $1-\frac{\delta}{2d^2}$,
    
    \begin{equation}\label{condition_smooth_4}
        \max_{\mathbf{x}\in\mathbb{X}}\left|\frac{\partial^2}{\partial x_i \partial x_j}\mu_{\mathcal{D}}(\mathbf{x})\right| \le \sqrt{2\log\left(\frac{2d^2}{\delta}\right)}\max_{\mathbf{x}\in\mathbb{X}}\sqrt{k^{\partial_{ij}}(\mathbf{x},\mathbf{x})}+L_{ij}
    \end{equation}
    Combine the result of Eq (\ref{condition_smooth_1}), Eq (\ref{condition_smooth_2}), Eq (\ref{condition_smooth_3}) and Eq (\ref{condition_smooth_4}), this Theorem is proved.
\end{proof}

\subsection{Local smoothness in standard derivation function}\label{sec:smooth_sigma}

In this subsection we attempt to build the smoothness theorem for standard derivation function $\sigma_{\mathcal{D}}(\mathbf{x})$. However, $\sigma_{\mathcal{D}}(\mathbf{x})$ is not L-smoothness under some common kernels such as Gaussian kernel or Matérn kernel with $\gamma=2.5$. However, we can proof that $\sigma_{\mathcal{D}}(\mathbf{x})$ may achieve a similar result with some small error term. 

Here we define an event $U_{L}=\{f(\cdot)|f(\cdot) \text{ is L-smooth}\}$, and $\sigma_{\mathcal{D}}(\mathbf{x}|L)=\sqrt{\frac{\pi}{2}}E(|f(\mathbf{x})-\mu_{\mathcal{D}}(\mathbf{x})|\big |\mathcal{D},f(\cdot)\in U_{L})$

\begin{lemma}\label{var_control_1}
For any $\mathbf{x}\in \mathbb{X}$ and dataset $\mathcal{D}$, we have
    \begin{equation*}
        \sigma_{\mathcal{D}}(\mathbf{x}|L)\le \sigma_{\mathcal{D}}(\mathbf{x}) + (\frac{1}{P(U_{L}|\mathcal{D})}-1)\max_{\mathbf{x}\in \mathbb{R}^{d}}\sqrt{k(\mathbf{x},\mathbf{x})}
    \end{equation*}
\end{lemma}

\begin{proof}
    According to the definition of $\sigma_L(\mathbf{x})$, we have
    \begin{align*}
        P(U_{L}|\mathcal{D})\sigma_{\mathcal{D}}(\mathbf{x}|L)+P(U_{L}^{c}|\mathcal{D})\sqrt{\frac{\pi}{2}}E(|f(\mathbf{x})-\mu_{\mathcal{D}}(\mathbf{x})|\big |\mathcal{D},f(\cdot)\in U_{L}^{c})=\sqrt{\frac{\pi}{2}}E(|f(\mathbf{x})-\mu_{\mathcal{D}}(\mathbf{x})|\big|\mathcal{D})=\sigma_{\mathcal{D}}(\mathbf{x})
    \end{align*}
    Thus 
    \begin{align*}
        \sigma_{\mathcal{D}}(\mathbf{x}|L)&\le \frac{1}{P(U_{L}|\mathcal{D})}\sigma_{\mathcal{D}}(\mathbf{x})
        \\
        &= \sigma_{\mathcal{D}}(\mathbf{x}) + (\frac{1}{P(U_{L}|\mathcal{D})}-1)\sigma_{\mathcal{D}}(\mathbf{x})
        \\
        &\le \sigma_{\mathcal{D}}(\mathbf{x}) + (\frac{1}{P(U_{L}|\mathcal{D})}-1)\max_{\mathbf{x}\in \mathbb{R}^{d}}\sqrt{k(\mathbf{x},\mathbf{x})}
    \end{align*}
\end{proof}

\begin{lemma}\label{var_control_2}
    If $P(U_{L}|\mathcal{D})>\frac{1}{2}$, there exist constants $c_1>1$ and $c_2>0$ independent of $\mathbf{x}$ and dataset $\mathcal{D}$ that satisfies
    \begin{equation*}
        \sigma_{\mathcal{D}}(\mathbf{x}) \le c_1\sigma_{\mathcal{D}}(\mathbf{x}|L)+c_2P(U_{L}^{c}|\mathcal{D})
    \end{equation*}

\end{lemma}

\begin{proof}
    According to the Markov inequality, we have
    \begin{equation*}
        P(|f(\mathbf{x})-\mu_{\mathcal{D}}(\mathbf{x})|>a\big |\mathcal{D}, f(\cdot)\in U_L)\le \frac{E(|f(\mathbf{x})-\mu_{\mathcal{D}}(\mathbf{x})|\big |\mathcal{D},f(\cdot)\in U_{L})}{a}=\sqrt{\frac{2}{\pi}}\frac{\sigma_{\mathcal{D}}(\mathbf{x}|L)}{a}
    \end{equation*}
    For any event $A$ and $B$, we have $P(A\cap B)=P(A)-P(A\cap B^{c})\ge P(A)-P(B^{c})$. Thus we have
    \begin{align*}
        P(|f(\mathbf{x})-\mu_{\mathcal{D}}(\mathbf{x})|>a\big |\mathcal{D}, f(\cdot)\in U_L)
        &=\frac{P(|f(\mathbf{x})-\mu_{\mathcal{D}}(\mathbf{x})|>a, f(\cdot)\in U_L\big |\mathcal{D})}{P(U_L|\mathcal{D})}\\
        &\ge \frac{P(|f(\mathbf{x})-\mu_{\mathcal{D}}(\mathbf{x})|>a\big |\mathcal{D})-P(U_L^c|\mathcal{D})}{P(U_L|\mathcal{D})}\\
    \end{align*}
    If we set $a=\sigma_{\mathcal{D}}(\mathbf{x})$, we will have the following inequality 
    \begin{align*}
        \sigma_{\mathcal{D}}(\mathbf{x}|L)&\ge a\sqrt{\frac{\pi}{2}}P(|f(\mathbf{x})-\mu_{\mathcal{D}}(\mathbf{x})|>a\big |\mathcal{D}, f(\cdot)\in U_L)
        \\
        &\ge \sqrt{\frac{\pi}{2}}\frac{2-2\Phi(1)-P(f(\cdot)\in U_L^c|\mathcal{D})}{P( U_L|\mathcal{D})}\sigma_{\mathcal{D}}(\mathbf{x})
        \\
        &\ge (2-2\Phi(1))\sqrt{\frac{\pi}{2}}\sigma_{\mathcal{D}}(\mathbf{x})-2\sqrt{\frac{\pi}{2}}\max_{\mathbf{x}\in \mathbb{X}}\sqrt{k(\mathbf{x},\mathbf{x})}P(U_L^c|\mathcal{D})
    \end{align*}

    So through the above analysis, we have
    \begin{equation*}
        \sigma_{\mathcal{D}}(\mathbf{x}) \le c_1\sigma_{\mathcal{D}}(\mathbf{x}|L)+c_2P(U_{L}^c|\mathcal{D})
    \end{equation*}
    where $c_1=\sqrt{\frac{2}{\pi}}\frac{1}{2-2\Phi(1)}$, $c_2=\frac{1}{1-\Phi(1)}\max_{\mathbf{x}\in \mathbb{X}}\sqrt{k(\mathbf{x},\mathbf{x})}$
\end{proof}

 \begin{lemma}\label{var_control_3}
    If $\mu_{\mathcal{D}}(\cdot)$ is $L_{\mu}$ smooth and $P(U_{L}|\mathcal{D})>\frac{1}{2}$, then we have 
     \begin{equation*}
         |\sigma_{\mathcal{D}}(\mathbf{x}_1|L)-\sigma_{\mathcal{D}}(\mathbf{x}_2|L)|
         \le \sqrt{\frac{\pi}{2}}\left(\|\mathbf{x}_1-\mathbf{x}_2\|_2\sqrt{\textbf{tr}(\nabla k_{\mathcal{D}}(\mathbf{x}_{2},\mathbf{x}_{2})\nabla^{T})}+\frac{3L+3L_{\mu}}{2}\|\mathbf{x}_1-\mathbf{x}_2\|_2^2\right)
     \end{equation*}
 \end{lemma}

 \begin{proof}
    
     \begin{align*}
         \sqrt{\frac{2}{\pi}}|\sigma_{\mathcal{D}}&(\mathbf{x}_1|L)-\sigma_{\mathcal{D}}(\mathbf{x}_2|L)|\\
         &=
         |E(|f(\mathbf{x}_1)-\mu_{\mathcal{D}}(\mathbf{x}_1)\big |\mathcal{D}, f(\cdot)\in U_{L})-E(|f(\mathbf{x}_2)-\mu_{\mathcal{D}}(\mathbf{x}_2)|\big |\mathcal{D},f(\cdot)\in U_{L})|
         \\
         &\le
         E(|f(\mathbf{x}_1)-\mu_{\mathcal{D}}(\mathbf{x}_1)-f(\mathbf{x}_2)+\mu_{\mathcal{D}}(\mathbf{x}_2)|\big |\mathcal{D},f(\cdot)\in U_{L})
     \end{align*}
    As $f(\cdot)\in U_{L}$ and $\mu_{\mathcal{D}}(\cdot)$ is $L_{\mu}$ smooth, 
     \begin{align*}
         f(\mathbf{x}_1)&\le f(\mathbf{x}_2)+<\nabla f(\mathbf{x}_2), \mathbf{x}_1-\mathbf{x}_2> + \frac{L}{2}\|\mathbf{x}_1-\mathbf{x}_2\|_2^2\\
        \mu_{\mathcal{D}}(\mathbf{x}_2)&\le \mu_{\mathcal{D}}(\mathbf{x}_1)+<\nabla \mu_{\mathcal{D}}(\mathbf{x}_1), \mathbf{x}_2-\mathbf{x}_1> + \frac{L_{\mu}}{2}\|\mathbf{x}_1-\mathbf{x}_2\|_2^2\\
        & \le \mu_{\mathcal{D}}(\mathbf{x}_1)+<\nabla \mu_{\mathcal{D}}(\mathbf{x}_2), \mathbf{x}_2-\mathbf{x}_1> + \frac{3L_{\mu}}{2}\|\mathbf{x}_1-\mathbf{x}_2\|_2^2
     \end{align*}

    \begin{align*}
        f(\mathbf{x}_1)-\mu_{\mathcal{D}}(\mathbf{x}_1)-f(\mathbf{x}_2)+\mu_{\mathcal{D}}(\mathbf{x}_2)&\le <\nabla f(\mathbf{x}_2)-\nabla \mu_{\mathcal{D}}(\mathbf{x}_2), \mathbf{x}_1-\mathbf{x}_2> + \frac{L+3L_{\mu}}{2}\|\mathbf{x}_1-\mathbf{x}_2\|_2^2
        \\
        &\le
        \|\nabla f(\mathbf{x}_2)-\nabla \mu_{\mathcal{D}}(\mathbf{x}_2)\|_2\|\mathbf{x}_1-\mathbf{x}_2\|_2+ \frac{L+3L_{\mu}}{2}\|\mathbf{x}_1-\mathbf{x}_2\|_2^2
    \end{align*}
    In a similar way, we can prove 
    \begin{align*}
        f(\mathbf{x}_1)-\mu_{\mathcal{D}}(\mathbf{x}_1)-f(\mathbf{x}_2)+\mu_{\mathcal{D}}(\mathbf{x}_2)
        \ge
        -\|\nabla f(\mathbf{x}_2)-\nabla \mu_{\mathcal{D}}(\mathbf{x}_2)\|_2\|\mathbf{x}_1-\mathbf{x}_2\|_2- \frac{3L+L_{\mu}}{2}\|\mathbf{x}_1-\mathbf{x}_2\|_2^2
    \end{align*}
    So through the above analysis, we have the following results

     \begin{align*}
         \sqrt{\frac{2}{\pi}}\sigma_{\mathcal{D}}&(\mathbf{x}_1|L)-\sigma_{\mathcal{D}}(\mathbf{x}_2|L)|\\
         &\le
         E(|f(\mathbf{x}_1)-\mu_{\mathcal{D}}(\mathbf{x}_1)-f(\mathbf{x}_2)+\mu_{\mathcal{D}}(\mathbf{x}_2)|\big |\mathcal{D},f(\cdot)\in U_{L})\\
        &\le 
         E(\|\nabla f(\mathbf{x}_2)-\nabla \mu_{\mathcal{D}}(\mathbf{x}_2)\|_2||\big |\mathcal{D},f(\cdot)\in U_{L})\|\mathbf{x}_1-\mathbf{x}_2\|_2 + \frac{3L+3L_{\mu}}{2}\|\mathbf{x}_1-\mathbf{x}_2\|_2^2
         \\
         &\le
         \frac{E(\|\nabla f(\mathbf{x}_2)-\nabla \mu_{\mathcal{D}}(\mathbf{x}_2)\|_2\big | \mathcal{D})}{P(U_{L}|\mathcal{D})}\|\mathbf{x}_1-\mathbf{x}_2\|_2+\frac{3L+3L_{\mu}}{2}\|\mathbf{x}_1-\mathbf{x}_2\|_2^2
        \\
        &\le
         \frac{\sqrt{E(\|\nabla f(\mathbf{x}_2)-\nabla \mu_{\mathcal{D}}(\mathbf{x}_2)\|_2^2|\mathcal{D})}}{P(U_{L}|\mathcal{D})}\|\mathbf{x}_1-\mathbf{x}_2\|_2+\frac{3L+3L_{\mu}}{2}\|\mathbf{x}_1-\mathbf{x}_2\|_2^2  
         \\
         &\le
         2\sqrt{E(\|\nabla f(\mathbf{x}_2)-\nabla \mu_{\mathcal{D}}(\mathbf{x}_2)\|_2^2|\mathcal{D})}\|\mathbf{x}_1-\mathbf{x}_2\|_2+\frac{3L+3L_{\mu}}{2}\|\mathbf{x}_1-\mathbf{x}_2\|_2^2
         \\
         &=\|\mathbf{x}_1-\mathbf{x}_2\|_2\sqrt{\textbf{tr}(\nabla k_{\mathcal{D}}(\mathbf{x}_{2},\mathbf{x}_{2})\nabla^{T})}+\frac{3L+3L_{\mu}}{2}\|\mathbf{x}_1-\mathbf{x}_2\|_2^2
     \end{align*}

     The  second to the last line is because of Cauchy-Schwarz inequality: $E(|X|)=E(|X|\cdot 1)\le \sqrt{E(X^2)} \cdot 1=\sqrt{E(X^2)}$.
 \end{proof}

\subsection{Some properties about approximate gradient descent}\label{sec:app_gd}

In this subsection we list properties about approximate gradient descent, which is mainly from Wu et al.\cite{wu2023behavior}. This part is also essential in our final convergence proof, as we will need to connect the previous approximate gradient descent method with our minimizing UCB to give the convergence speed of gradient. 

We first borrow the definition of Error function from Wu's work. The Error function measures the maximum reduction of uncertainty about the gradient estimation at $\mathbf{x}=0$ when there are $b$ data points $\mathbf{Z}$ without any extra dataset:

\begin{definition_main}{(Error function)}
Given input dimensionality $d$, kernel $k$ and noise standard deviation $\sigma$, we define the following error function:
\begin{equation}
    E_{d,k,\sigma}(b)=\inf_{\mathcal{Z}\in\mathbb{R}^{b\times d}}\textbf{tr}(\nabla k(\mathbf{0},\mathbf{0})\nabla^{T}-\nabla k(\mathbf{0},\mathbf{Z})(k(\mathbf{Z},\mathbf{Z})+\sigma^2I)^{-1}k(\mathbf{0},\mathbf{Z})\nabla^{T})
\end{equation}
\end{definition_main}

This Error function actually bounds the variance of the estimated gradient, which can be seen in the following lemma:
\begin{lemma}\label{error_func}
    In the $t^{th}$ iterations in MinUCB, we have
    \begin{equation*}
        \textbf{tr}(\nabla k_{\mathcal{D}_{t}}(\mathbf{x}_{t},\mathbf{x}_{t})\nabla^{T})\le  E_{d,k,\sigma}(b^{(2)}_{t})
    \end{equation*}
\end{lemma}

\begin{proof}
    In the $t^{th}$ step of MinUCB, the sampled candidates can be divided into two parts. One of them are sampled through local exploration acquisition function:
    \begin{equation*}
        \mathbf{X}=\arg\min_{\mathbf{Z}}\alpha_{trace}(\mathbf{x}_{t},\mathbf{Z})
    \end{equation*}
    Suppose their corresponding label is $\mathbf{y}$, and we set $\mathcal{D}_{2t}=(\mathbf{X},\mathbf{y})$, then we can obtain:
    \begin{equation*}
        \textbf{tr}(\nabla k_{\mathcal{D}_{t}}(\mathbf{x}_{t},\mathbf{x}_{t})\nabla^{T})\le\textbf{tr}(\nabla k_{\mathcal{D}_{2t}}(\mathbf{x}_{t},\mathbf{x}_{t})\nabla^{T})
    \end{equation*}
    The above inequality is mainly because the $\mathcal{D}_{2t}$ is the subset of $\mathcal{D}_{t}$. Then through
the same analysis with Lemma 8 in Wu et al.\cite{wu2023behavior}, we can prove that 
\begin{equation*}
    \textbf{tr}(\nabla k_{\mathcal{D}_{2t}}(\mathbf{x}_{t},\mathbf{x}_{t})\nabla^{T})\le E_{d,k,\sigma}(b^{(2)}_{t})
\end{equation*}
which complete our proof.
\end{proof}

To finish the convergence proof of MinUCB, we need the following lemmas from Wu et al.\cite{wu2023behavior}, this lemmas give the upper bound of the estimation error of gradient, and the upper bound of Error function under two common kernels RBF kernel and Matérn kernel with $\gamma=2.5$.

\begin{lemma} \label{gradient_est_error}
    For any $0\le\delta\le 1$, let $C_t=2\log\left(\frac{\pi^2 t^2}{6\delta}\right)$. Then the inequalities
    \begin{equation}
        \|\nabla f(\mathbf{x}_{t})-\nabla \mu_{\mathcal{D}_{t}}(\mathbf{x}_{t})\|_2^2\le C_{t}\textbf{tr}(\nabla k_{\mathcal{D}_{t}}(\mathbf{x}_{t},\mathbf{x}_{t})\nabla^{T})
    \end{equation}
\end{lemma}

\begin{lemma} \label{error_RBF}
    let $k(\mathbf{x}_1,\mathbf{x}_2) = \exp(-\frac{\|\mathbf{x}_1,\mathbf{x}_2\|_2}{2})$ be the RBF kernel. We have
    \begin{equation*}
        E_{d,k,\sigma}(2md)=O(\sigma d m^{-\frac{1}{2}})
    \end{equation*}
\end{lemma}

\begin{lemma}\label{error_matern}
    let $k(\cdot,\cdot)$ be the Matérn kernel. Then we have
    \begin{equation*}
        E_{d,k,\sigma}(2md)=O(\sigma d m^{-\frac{1}{2}})
    \end{equation*}
\end{lemma}

\subsection{Convergence proof of MinUCB}
In this part we will show the convergence of MinUCB, which will use all the results in previous Subsections. The difficulty in proving the convergence is to build the relationship with previous approximate gradient methods and minimizing UCB. In this proof we try to connect the function value on the gradient descent point and minimizing UCB point, and will need some smoothness properties on the Gaussian process, mean function and standard derivation function, which are provided in Subsection \ref{sec:smooth_gp_mu} and \ref{sec:smooth_sigma}. To give the accurate convergence rate, we need the upper bound on the Error function, which is provided in Subsection \ref{sec:app_gd}. We first give the basic convergence theorem on the gradient for MinUCB:

\begin{theorem}\label{main_conclusion}
        Suppose $f$ is sampled from a zero mean Gaussian process with a continuously differentiable convariance function $k(\cdot,\cdot)$, and $k(\cdot,\cdot)$ satisfies Assumption 1. Then after $t$ 
        iterations of MinUCB algorithm, with the batch size $b_{t}^{(1)}$ and $b_{t}^{(2)}$, it satisfies that
\begin{align*}
    \min_{T/2\le t\le T} \|\nabla f(\mathbf{x}_t)\|_2 &\le \frac{1}{\sqrt{\eta_{T}}}\sqrt{ \frac{8}{T}\sum_{t=1}^{T} \frac{\Tilde{\beta}_t\sigma}{\sqrt{b_t^{(1)}}} 
    + \frac{8\pi}{T}\sum_{t=1}^{T}\Tilde{\beta}_t^2 \eta_t E_{d,k,\sigma}(b_{t}^{(2)}) + O(\frac{1}{T}\log\frac{1}{\delta})}
    \\
    &\quad+\frac{1}{\sqrt{\eta_{T}}}\Tilde{\beta}_{T/2}\sqrt{\frac{\pi}{2}}\sqrt{E_{d,k,\sigma}(b_{T/2}^{(2)})}
\end{align*}
where $\Tilde{\beta}_{t}$ and $\eta_{t}$ are both decreasing sequence. They satisfies $\Tilde{\beta}_{t}=O(\beta_{t})$ and $\frac{1}{\eta_{t}}=O(d\sqrt{\log\frac{t^2d^2}{\delta}}+d^{\frac{3}{2}})$, and $\beta_{t}=\sqrt{2\log\frac{\pi^2t^2}{\delta}}$ 
\end{theorem}

\begin{proof}

According to the definition of $\mathbf{x}_{t+1}$, $\mathbf{x}_{t+1}=\arg\max_{\mathbf{x}\in\mathbb{X}}\mu_{\mathcal{D}_t}(\mathbf{x}) + \beta_{t}\sigma_{\mathcal{D}_t}(\mathbf{x})$. 
\begin{equation}\label{basic_eq}
    \begin{aligned}
    f(\mathbf{x}_{t+1}) & 
    \le \min_{x\in \mathbb{R}^{d}} \mu_{\mathcal{D}_t}(\mathbf{x}) + \beta_{t}\sigma_{\mathcal{D}_t}(\mathbf{x})
    \\
    &\le \mu_{\mathcal{D}_t}(\hat{\mathbf{x}}_{t+1}) + \beta_{t}\sigma_{\mathcal{D}_t}(\hat{\mathbf{x}}_{t+1})
\end{aligned}
\end{equation}

Where $\hat{\mathbf{x}}_{t+1}$ is a special point $\hat{\mathbf{x}}_{t+1}=\mathbf{x}_{t}-\eta_{t}\nabla \mu_{\mathcal{D}_t}(\mathbf{x}_{t})$, which is a pseudo gradient descent step. We will use this $\hat{\mathbf{x}}_{t+1}$ to build the connection between gradient descent and minimizing UCB. The $\beta_{t}$ here is carefully chosen as $\beta_{t}=\sqrt{2\log\frac{\pi^2t^2}{\delta}}$ to guarantee  
\begin{equation*}
    \sum_{t=1}^{\infty}P\big(f(\mathbf{x}_{t})\le \mu_{\mathcal{D}_{t-1}}(\mathbf{x}_{t}) + \beta_{t}\sigma_{\mathcal{D}_{t-1}}(\mathbf{x}_{t})\big)\ge 1-\frac{\delta}{6}
\end{equation*}

Now we try to give the relationship between the UCB bound on $\mathbf{x}_{t}$ and $\hat{\mathbf{x}}_{t+1}$. Suppose the mean function $\mu_{\mathcal{D}_{t}}(\cdot)$ is $L_{\mu_{t}}-$smoothness (this coefficient will be given in the subsequent parts), we will have 

\begin{equation}\label{fin_mu_part}
\begin{aligned}
    \mu_{\mathcal{D}_t}(\hat{\mathbf{x}}_{t+1}) &\le  \mu_{\mathcal{D}_t}(\mathbf{x}_t)+<\nabla \mu_{\mathcal{D}_t}(\mathbf{x}_{t}),\hat{\mathbf{x}}_{t+1}-\mathbf{x}_t> + \frac{L_{\mu_t}}{2}\|\hat{\mathbf{x}}_{t+1}-\mathbf{x}_t\|_2^2\\
    &\le
    \mu_{\mathcal{D}_t}(\mathbf{x}_t)- \eta_{t}\|\nabla \mu_{\mathcal{D}_t}(\mathbf{x}_{t})\|_2^2 + \frac{L_{\mu_t}}{2}\eta^2_{t}\|\nabla \mu_{\mathcal{D}_t}(\mathbf{x}_{t})\|_2^2
\end{aligned}
\end{equation}

Here we apply the results in Subsection \ref{sec:smooth_sigma}, and give the upper bound for the $\sigma_{\mathcal{D}_t}(\hat{\mathbf{x}}_{t+1})$. To simplify the symbols, we define $\textbf{tr}(\nabla k_{\mathcal{D}_{t}}(\mathbf{x}_{t},\mathbf{x}_{t})\nabla^{T})=\gamma_{t}$.
\begin{equation}\label{fin_sigma_part}
\begin{aligned}
    \sigma_{\mathcal{D}_t}(\hat{\mathbf{x}}_{t+1})
    \le& c_1\sigma_{\mathcal{D}_t}(\hat{\mathbf{x}}_{t+1}|L_{t})+c_2P( U_{L_{t}}^c|\mathcal{D}_t)
    \\
    \le& c_1\sigma_{\mathcal{D}_t}(\mathbf{x}_{t}|L_{t})+c_1\big|\sigma_{\mathcal{D}_{t}}(\hat{\mathbf{x}}_{t+1}|L_{t})-\sigma_{\mathcal{D}_t}(\mathbf{x}_{t}|L_{t})\big|+c_2P( U_{L_{t}}^c|\mathcal{D}_t)
    \\
    \le& c_1\sigma_{\mathcal{D}_t}(\mathbf{x}_{t}|L_{t}) + c_1\sqrt{\frac{\pi}{2}} \eta_{t}\sqrt{\gamma_t}\|\nabla \mu_{\mathcal{D}_t}(\mathbf{x}_{t})\|_2 + c_1\sqrt{\frac{\pi}{2}} \frac{3L_t+3L_{\mu_t}}{2}\eta_{t}^2\|\nabla \mu_{\mathcal{D}_t}(\mathbf{x}_{t})\|_2^2+c_2P(  U_{L_{t}}^c|\mathcal{D}_t)
    \\
    \le& c_1\sigma_{\mathcal{D}_t}(\mathbf{x}_{t}) + c_1\sqrt{\frac{\pi}{2}} \eta_{t}\sqrt{\gamma_t}\|\nabla \mu_{\mathcal{D}_t}(\mathbf{x}_{t})\|_2 + c_1\sqrt{\frac{\pi}{2}} \frac{3L_t+3L_{\mu_t}}{2}\eta_{t}^2\|\nabla \mu_{\mathcal{D}_t}(\mathbf{x}_{t})\|_2^2
    \\
    &+c_2P(  U_{L_{t}}^c|\mathcal{D}_t)+c_1(\frac{1}{P(U_{L_{t}}|\mathcal{D})}-1)\max_{\mathbf{x}\in \mathbb{X}}\sqrt{k(\mathbf{x},\mathbf{x})}
\end{aligned}
\end{equation}
    where the first line apply Lemma \ref{var_control_1}. The third line use Lemma \ref{var_control_3} and the last line is achieved through Lemma \ref{var_control_2}. We now try to give the coefficient of $L_{\mu_t}$ and $L_{t}$, and these coefficients are all based on the smoothness coefficient of Gaussian process. First, according to Theorem \ref{GP_smooth}, with the probability of at least $1-\frac{\delta}{6}$, for any $i,j=1,...,d$, we have 
        \begin{equation}
        L_{ij}=\sup_{\mathbf{x}\in\mathbb{X}}\left|\frac{\partial^2}{\partial x_i \partial x_j}f(\mathbf{x})\right|\le \sqrt{2\log\left(\frac{12d^2}{\delta}\right)}\max_{\mathbf{x}\in\mathbb{X}}\sqrt{k^{\partial_{ij}}(\mathbf{x},\mathbf{x})}+12\sqrt{6d}\max\left\{\max_{\mathbf{x}\in\mathcal{X}}\sqrt{k^{\partial_{ij}}(\mathbf{x},\mathbf{x})},\sqrt{rL^{\partial_{ij}}_k}\right\}
    \end{equation}
    
     If we carefully choose $L_{\mu_{t}}$ and $L_{t}$ as 

\begin{equation*}
        L_{\mu_{t}}=\sqrt{\sum_{i,j}V_{t,ij}^2} \qquad L_{t}=\sqrt{\sum_{i,j}\Tilde{U}_{t,ij}^2}
    \end{equation*}
    where
    \begin{align*}
        V_{t,ij}&=\sqrt{2\log\left(\frac{\pi^2 t^2 d^2}{\delta}\right)}\max_{\mathbf{x}\in\mathbb{X}}\sqrt{k^{\partial_{ij}}(\mathbf{x},\mathbf{x})}+L_{ij}
        \\
        \Tilde{U}_{t,ij}&=L_{ij}+2\sqrt{2\log\left(\frac{4\pi^2 t^2d^2}{\delta}\right)}\max_{\mathbf{x}\in\mathbb{X}}\sqrt{k^{\partial_{ij}}(\mathbf{x},\mathbf{x})}+12\sqrt{6d}\max\left\{\max_{\mathbf{x}\in\mathcal{X}}\sqrt{k^{\partial_{ij}}(\mathbf{x},\mathbf{x})},\sqrt{rL^{\partial_{ij}}_k}\right\}
    \end{align*}
     According to Theorem \ref{mu_L_smooth} and \ref{conditional_L_smooth}, this guarantees that 
     
     \begin{align*}
         &\sum_{t=1}^{\infty}P(\mu_{\mathcal{D}_{t}}(\cdot) \text{ is }L_{\mu_{t}}-smooth)\ge1-\frac{\delta}{6}\\
         &\sum_{t=1}^{\infty}P(U_{L_{t}}|\mathcal{D}_{t})\ge1-\frac{\delta}{6}
     \end{align*}

     and we also have $P(U_{L_{t}}^{c}|D_{t})\le \frac{6\delta}{\pi^2 t^2}$. Thus for the standard derivation term, their approximate error term is tend to be a very small value:
    \begin{equation}\label{fin_minior_error}
        c_2P( f(\cdot)\in U_{L_{t}}^c|\mathcal{D}_t)+c_1(\frac{1}{P(U_{L_{t}}|\mathcal{D})}-1)\max_{\mathbf{x}\in \mathbb{R}^{d}}\sqrt{k(\mathbf{x},\mathbf{x})} = O\left(\frac{1}{t^2}\right)
    \end{equation}

If we combine Eq. (\ref{basic_eq}), Eq. (\ref{fin_mu_part}), Eq. (\ref{fin_sigma_part}) and Eq. (\ref{fin_minior_error}), and we also set $\Tilde{\beta}_{t} = c_1\beta_t$ and $\eta_{t}= \sqrt{\frac{2}{\pi}}\frac{1}{6\beta_{t}c_1(L_{t}+L_{\mu_{t}})}$, we can obtain an upper bound for $f(\mathbf{x}_{t+1})$, and this upper bound is related to the gradient:
\begin{align*}
    f(\mathbf{x}_{t+1}) &\le \mu_{\mathcal{D}_{t}}(\mathbf{x}_t) + \Tilde{\beta}_{t} \sigma_{\mathcal{D}_{t}}(\mathbf{x}_t) -\frac{1}{2}\eta_t\|\nabla \mu_{\mathcal{D}_t}(\mathbf{x}_{t})\|_2^2  
    +  \Tilde{\beta}_{t}\sqrt{\frac{\pi}{2}}\eta_{t}\sqrt{\gamma_{t}}\|\nabla \mu_{\mathcal{D}_t}(\mathbf{x}_{t})\|_2 + \Tilde{\beta}_{t} O\left(\frac{1}{t^2}\right)\\
    & \le \mu_{\mathcal{D}_{t}}(\mathbf{x}_t) + \Tilde{\beta}_{t} \sigma_{\mathcal{D}_{t}}(\mathbf{x}_t)-\frac{1}{2}\eta_t\|\nabla f(\mathbf{x}_t)\|_2^2 + \frac{1}{2}\eta_t(\|\nabla f(\mathbf{x}_t)\|_2^2- \|\nabla \mu_{\mathcal{D}_t}(\mathbf{x}_{t})\|_2^2 )
    \\
    &\quad +  \Tilde{\beta}_{t}\sqrt{\frac{\pi}{2}}\eta_{t}\sqrt{\gamma_{t}}\|\nabla \mu_{\mathcal{D}_t}(\mathbf{x}_{t})\|_2 + \Tilde{\beta}_{t} O\left(\frac{1}{t^2}\right)
\end{align*}
Through simple analysis, we have
\begin{align*}
    \|\nabla f(\mathbf{x}_t)\|_2^2- \|\nabla \mu_{\mathcal{D}_t}(\mathbf{x}_{t})\|_2^2 
    &= (\|\nabla f(\mathbf{x}_t)\|_2- \|\nabla \mu_{\mathcal{D}_t}(\mathbf{x}_{t})\|_2)(\|\nabla f(\mathbf{x}_t)\|_2+ \|\nabla \mu_{\mathcal{D}_t}(\mathbf{x}_{t})\|_2)
    \\
    &= -(\|\nabla f(\mathbf{x}_t)\|_2- \|\nabla \mu_{\mathcal{D}_t}(\mathbf{x}_{t})\|_2)^2+2(\|\nabla f(\mathbf{x}_t)\|_2- \|\nabla \mu_{\mathcal{D}_t}(\mathbf{x}_{t})\|_2)\|\nabla f(\mathbf{x}_t)\|_2
    \\
    &\le 2\|\nabla f(\mathbf{x}_t)-\nabla \mu_{\mathcal{D}_t}(\mathbf{x}_{t})\|_2\|\nabla f(\mathbf{x}_t)\|_2
\end{align*}
   and 
\begin{align*}
    \|\nabla \mu_{\mathcal{D}_t}(\mathbf{x}_{t})\|_2 &\le \|\nabla \mu_{\mathcal{D}_t}(\mathbf{x}_{t})-\nabla f(\mathbf{x}_t)\|_2+\|\nabla f(\mathbf{x}_t)\|_2
\end{align*}

Thus if we consider the Lemma \ref{gradient_est_error}, we can transform the previous upper bound into the following result with probability $1-\frac{\delta}{6}$ for any $t>1$

\begin{align*}
    f(\mathbf{x}_{t+1}) &\le \mu_{\mathcal{D}_{t}}(\mathbf{x}_t) + \Tilde{\beta}_{t} \sigma_{\mathcal{D}_{t}}(\mathbf{x}_t) -\frac{1}{2}\eta_t\|\nabla f(\mathbf{x}_t)\|_2^2 
    + 2\Tilde{\beta}_{t}\sqrt{\frac{\pi}{2}}\eta_t\sqrt{\gamma_{t}}\|\nabla f(\mathbf{x}_t)\|_2 
    \\
    &\quad + \Tilde{\beta}_{t}^2 \sqrt{\frac{\pi}{2}}\eta_t \gamma_t + \Tilde{\beta}_{t} O\left(\frac{1}{t^2}\right)  
\end{align*}

Based on the above result, we can further obtain the following results with probability $1-\frac{\delta}{6}$ for any $t>1$

\begin{align*}
    f(\mathbf{x}_{t+1}) -  f(\mathbf{x}_{t})&\le \mu_{\mathcal{D}_{t}}(\mathbf{x}_t) + \Tilde{\beta}_{t} \sigma_{\mathcal{D}_{t}}(\mathbf{x}_t)-  f(\mathbf{x}_{t}) -\frac{1}{2}\eta_t\|\nabla f(\mathbf{x}_t)\|_2^2 
    + 2\Tilde{\beta}_{t}\sqrt{\frac{\pi}{2}}\eta_t\sqrt{\gamma_{t}}\|\nabla f(\mathbf{x}_t)\|_2 
    \\
    &\quad+ \Tilde{\beta}_{t}^2 \sqrt{\frac{\pi}{2}}\eta_t \gamma_t + \Tilde{\beta}_{t} O(\frac{1}{t^2})  
    \\
    &\le 2\Tilde{\beta}_{t} \sigma_{\mathcal{D}_{t}}(\mathbf{x}_t) -\frac{1}{2}\eta_t\|\nabla f(\mathbf{x}_t)\|_2^2 
    + 2\Tilde{\beta}_{t}\sqrt{\frac{\pi}{2}}\eta_t\sqrt{\gamma_{t}}\|\nabla f(\mathbf{x}_t)\|_2 + \Tilde{\beta}_{t}^2 \sqrt{\frac{\pi}{2}}\eta_t \gamma_t + \Tilde{\beta}_{t} O\left(\frac{1}{t^2}\right) 
\end{align*}

Based on the above inequation, after $t$ steps in MinUCB algorithm, we can have the following result with probability of at least $1-\delta$:

\begin{align*}
    f(\mathbf{x}_{T}) - f(\mathbf{x}_{0}) 
    &\le  -\frac{1}{2}\sum_{t=1}^{T} \eta_{t}\left[\|\nabla f(\mathbf{x}_t)\|_2^2-4\Tilde{\beta}_{t}\sqrt{\frac{\pi}{2}}\eta_t\sqrt{\gamma_{t}}\|\nabla f(\mathbf{x}_t)\|_2\right]
    +2\sum_{t=1}^{T} \Tilde{\beta}_{t} \sigma_{\mathcal{D}_{t}}(\mathbf{x}_t) 
    \\
    &\quad+ \sum_{t=1}^{T}\Tilde{\beta}_{t}^2 \sqrt{\frac{\pi}{2}}\eta_t \gamma_t + O(\log\frac{1}{\delta})
\end{align*}

By organizing the above results, we can obtain

\begin{align*}
    \frac{1}{2}\sum_{t=1}^{T} \eta_{t}\left[\|\nabla f(\mathbf{x}_t)\|_2^2-4\Tilde{\beta}_{t}\sqrt{\frac{\pi}{2}}\sqrt{\gamma_{t}}\|\nabla f(\mathbf{x}_t)\|_2\right]&\le f(\mathbf{x}_{0})-f(\mathbf{x}_{T}) +2\sum_{t=1}^{T} \Tilde{\beta}_{t} \sigma_{\mathcal{D}_{t}}(\mathbf{x}_t) 
    \\
    &\quad+ \sum_{t=1}^{T}\Tilde{\beta}_{t}^2 \sqrt{\frac{\pi}{2}}\eta_t \gamma_t + O(\log\frac{1}{\delta})
    \\
    \frac{1}{2}\sum_{t=1}^{T} \eta_{t}\left[\|\nabla f(\mathbf{x}_t)\|_2-2\Tilde{\beta}_{t}\sqrt{\frac{\pi}{2}}\sqrt{\gamma_{t}}\right]^2&\le V^{*} +2\sum_{t=1}^{T} \Tilde{\beta}_{t} \sigma_{\mathcal{D}_{t}}(\mathbf{x}_t) 
    + 2\pi\sum_{t=1}^{T}\Tilde{\beta}_{t}^2 \eta_t \gamma_t + O(\log\frac{1}{\delta})
    \\
    \frac{1}{2}\sum_{t=T/2}^{T} \eta_{t}\left[\|\nabla f(\mathbf{x}_t)\|_2-2\Tilde{\beta}_{t}\sqrt{\frac{\pi}{2}}\sqrt{\gamma_{t}}\right]^2&\le V^{*} +2\sum_{t=1}^{T} \Tilde{\beta}_{t} \sigma_{\mathcal{D}_{t}}(\mathbf{x}_t) 
    + 2\pi\sum_{t=1}^{T}\Tilde{\beta}_{t}^2 \eta_t \gamma_t + O(\log\frac{1}{\delta})
\end{align*}

Based on this inequality we can directly give the upper bound for the norm of gradient:

\begin{align*}
    \min_{T/2\le t\le T} \eta_{t}\left[\|\nabla f(\mathbf{x}_t)\|_2^2-\Tilde{\beta}_{t}\sqrt{\frac{\pi}{2}}\eta_t\sqrt{\gamma_{t}}\right]^2 &\le \frac{4}{T} V^* + \frac{8}{T}\sum_{t=1}^{T} \Tilde{\beta}_{t} \sigma_{\mathcal{D}_{t}}(\mathbf{x}_t) 
    + \frac{4\pi}{T}\sum_{t=1}^{T}\Tilde{\beta}_{t}^2 \eta_t \gamma_t + O(\frac{1}{T}\log\frac{1}{\delta})
\end{align*}
\begin{align}\label{base_convergence}
    \min_{T/2\le t\le T} \|\nabla f(\mathbf{x}_t)\|_2 &\le \frac{1}{\sqrt{\eta_{T}}}\sqrt{ \frac{8}{T}\sum_{t=1}^{T} \Tilde{\beta}_{t} \sigma_{\mathcal{D}_{t}}(\mathbf{x}_t) 
    + \frac{8\pi}{T}\sum_{t=1}^{T}\Tilde{\beta}_{t}^2 \eta_t \gamma_t + O(\frac{1}{T}\log\frac{1}{\delta})}
    \\
    &\quad+\frac{1}{\sqrt{\eta_{T}}}\Tilde{\beta}_{T/2}\sqrt{\frac{\pi}{2}}\sqrt{\gamma_{T/2}}
\end{align}

Notice that in each step of MinUCB, we do the repeated sampling at point $\mathbf{x}_{t}$, we denote $\mathcal{D}_{t}^{1}=\{(\mathbf{x}_{t},y^{1}_{t}),...,(\mathbf{x}_{t},y^{b_{t}^{(1)}}_{t})\}$ as this resampling set. We will have
\begin{equation*}
    \sigma_{\mathcal{D}_{t}}(\mathbf{x}_t)\le \frac{\sigma}{\sqrt{b_{t}^{(1)}}}
\end{equation*}
Through Lemma \ref{error_func}, we can achieve
\begin{equation*}
    \gamma_{t}\le  E_{d,k,\sigma}(b^{(2)}_{t})
\end{equation*}
We combine these two value into Eq. (\ref{base_convergence}), which proves this theorem.

\end{proof}

\setcounter{theorem}{0}

\begin{theorem}
     Suppose $f$ is sampled from a zero mean Gaussian process with a continuously differentiable convariance function $k(\cdot,\cdot)$, then if the kernel is Gaussian kernel or Matérn kernel with $\gamma=2.5$, and satisfy $\beta_{t}=\sqrt{2\log\frac{\pi^2t^2}{\delta}}$ , batch size
		\begin{equation*}
                  b_{t}^{(1)}=
			\begin{cases}
				\log^{2}t\\
				t\\
				t^2
			\end{cases}
            \qquad \text{and} \qquad
			b_{t}^{(2)}=
			\begin{cases}
				d\log^{2}t\\
				dt\\
				dt^2
			\end{cases}
		\end{equation*}
		 MinUCB will achieve the convergence rate of
		\begin{equation*}
			\min_{T/2\le i\le T}\|\nabla f(x_i)\|_2^2\le
			\begin{cases}
				O(\sigma d^{\frac{3}{2}}T^{-1}\log^\frac{3}{2} \frac{d^2T^2}{\delta})+O(\sigma d^{2}) \\
				O(\sigma d^{2} T^{-\frac{1}{2}}\log^\frac{5}{2} \frac{d^2T^2}{\delta})=O(\sigma d^{\frac{9}{4}}n^{-\frac{1}{4}}\log^{\frac{5}{2}} \frac{dn}{\delta})\\
				O(\sigma d^{2} T^{-1}\log^\frac{5}{2} \frac{d^2T^2}{\delta}))=O(\sigma d^{\frac{7}{3}}n^{-\frac{1}{3}}\log^{\frac{5}{2}} \frac{dn}{\delta}))
			\end{cases}
		\end{equation*}
\end{theorem}

\begin{proof}
   According to the proof of Theorem \ref{main_conclusion}, we have
   \begin{equation}
       \frac{1}{\eta_{t}}=O(d\sqrt{\log\frac{t^2d^2}{\delta}}+d^{3/2})
   \end{equation}
 
    Through Lemma \ref{error_RBF} and \ref{error_matern}, we can bound $E_{d,k,\sigma}(b^{(2)}_{t})$ as $E_{d,k,\sigma}(b^{(2)}_{t})=O\left( \frac{\sigma\sqrt{b^{(2)}_{t}}}{d}\right)$. Thus if we combline the result of Theorem \ref{main_conclusion} and above error bounds, we will proof the results.

\end{proof}

\subsection{Theoretical property of LA-MinUCB}
In this section we will provide the theoretical property for our LA-MinUCB algorithm, that LA-MinUCB is Bayesian optimal if there is only one iteration left. Our result is proved in a similar way with the Proposition 2 in Wu et al. \cite{wu2017bayesian}.
\begin{theorem}
    If only one iteration is left and we can observe the function value through sampling, then the local exploration in LA-MinUCB is Bayes-optimal among all feasible policies.
\end{theorem}

\begin{proof}
    Suppose that we are given a budget of $N_{max}$ samples, i.e. we may run the algorithm for N iterations. Thus, letting $\Pi$ be the set of feasible policies $\pi$, we can formulate our problem as follows:
    \begin{equation*}
        \inf_{\pi\in\Pi} \mathbb{E}^{\pi}\left[ \min_{\mathbf{x}\in\mathbb{X}} \mu_{\mathcal{D}_{N}}(\mathbf{x})+\beta_{N}  \sigma_{\mathcal{D}_{N}}(\mathbf{x})   \right]
    \end{equation*}
    We analyze this problem under the DP framework. We define our state space as $S^{n}=(\mu_{\mathcal{D}_{n}}(\cdot),k_{\mathcal{D}_{n}}(\cdot,\cdot))$ after iteration $n$ as it completely characterizes our belief on $f$. Under the DP framework, we will define the value function $V^{n}$ as follows:
    \begin{equation*}
        V^{n}(s)\coloneqq \inf_{\pi\in\Pi} \mathbb{E}^{\pi}\left[ \min_{\mathbf{x}\in\mathbb{X}} \mu_{\mathcal{D}_{N}}(\mathbf{x})+\beta_{N}  \sigma_{\mathcal{D}_{N}}(\mathbf{x})| S^{n}=s   \right]
    \end{equation*}
\end{proof}
for every $s=(\mu,k)$. The Bellman equation tells us that the value function can be written recursively as
\begin{equation*}
    V^{n}(s) = \min_{\mathbf{Z}\in \mathbb{R}^{b_{n}}\times d} Q^{n}(s,\mathbf{Z})
\end{equation*}
where
\begin{equation*}
    Q^{n}(s,\mathbf{Z})=\mathbb{E}\left[V^{n+1}(S^{n+1})|S^{n+1}, \mathbf{X}=\mathbf{Z}\right]
\end{equation*}
and $\mathbf{X}$ is the new input in the $n+1^{th}$ step. At the same time, we also know that any policy $\pi^{*}$ whose decisions satisfy
\begin{equation}\label{LA_optimal}
    \mathbf{Z}^{\pi^*,n}(s)\in \arg\min_{\mathbf{Z}\in \mathbb{R}^{b_{n}}\times d}Q^{n}(s,\mathbf{Z})
\end{equation}
is optimal. If we were to stop at iteration $n+1$, then $V^{n+1}(S^{n+1})=\min_{\mathbf{x}\in\mathbb{X}} \mu_{\mathcal{D}_{n+1}}(\mathbf{x})+\beta_{n+1}  \sigma_{\mathcal{D}_{n+1}}(\mathbf{x})$ and Eq. (\ref{LA_optimal}) reduces to
\begin{align*}
    \mathbf{Z}^{\pi^*,n}(s)&\in \arg\min_{\mathbf{Z}\in \mathbb{R}^{b_{n}}\times d}\mathbb{E}\left[\min_{\mathbf{x}\in\mathbb{X}} \mu_{\mathcal{D}_{n+1}}(\mathbf{x})+\beta_{n+1}  \sigma_{\mathcal{D}_{n+1}}(\mathbf{x})|S^{n}=s, \mathbf{X}=\mathbf{Z}\right]
\end{align*}
    which is exactly the result of local exploration acquisition function of LA-MinUCB. This proves that  the local exploration acquisition function in LA-MinUCB is one-step Bayes-optimal.

\end{document}